\documentclass{article}

\PassOptionsToPackage{bookmarks=true}{hyperref}
\usepackage[table]{xcolor}

\usepackage[final]{corl_2026} % Uncomment for the camera-ready ``final'' version.

\usepackage{times}

\usepackage{amsmath}
\usepackage{amssymb}
\usepackage{amsthm}
\usepackage[numbers]{natbib}
\usepackage{multicol}
\usepackage{cleveref}
\usepackage{booktabs}
\usepackage{graphicx} % Required for inserting images
\usepackage{tabularx}
\usepackage{array}
\newcolumntype{Y}{>{\centering\arraybackslash}X}
\usepackage{multirow}
\usepackage{enumitem}
\usepackage{listings}

\usepackage{colortbl}
\usepackage[frozencache,cachedir=.]{minted}
\usepackage{xspace}

\usepackage[english]{babel}
\usepackage[utf8]{inputenc}
\usepackage{algorithm}
\usepackage[noend]{algpseudocode}
\usepackage{mathtools}

\usepackage{bm}
\usepackage{wrapfig}
\usepackage{tikz}
\usepackage{pgfplots}
\pgfplotsset{compat=newest}
\usepackage{verbatim}
\usepackage{flushend}
\usepackage[font=footnotesize]{caption}
\usepackage{subcaption}
\usepackage{makecell}

\newcommand{\cD}{\mathcal{D}}

\newcommand{\cL}{\mathcal{L}}

\newcommand{\cU}{\mathcal{U}}

\newcommand{\cX}{\mathcal{X}}

\newcommand{\EE}{\mathbb{E}}

\newcommand{\RR}{\mathbb{R}}

\newcommand{\defeq}{\vcentcolon=}

\newtheorem{theorem}{Theorem}

\newtheorem*{remark*}{Remark}

\newtheorem{proposition}{Proposition}

\colorlet{revision_color}{black}

\definecolor{miracleolive}{RGB}{85,107,47} % olive drab-ish
\newcommand{\miracle}[1]{\textcolor{miracleolive}{#1}}

\definecolor{respplanned}{HTML}{0000FF}   % blue
\definecolor{respuncertain}{HTML}{CC0000} % red
\definecolor{resp}{HTML}{087F5B} % red

\definecolor{revrcolor}{HTML}{8A6D00}     % dark gold

\newcommand{\camrdy}[1]{{#1}}

\title{$\lambda$-Reachability: Geometric-Horizon Safety Bellman Equations for Humanoid Safety}

\author{
  Rui Chen\textsuperscript{1} \quad
  Shangtao Li\textsuperscript{2} \quad
  Yifan Sun\textsuperscript{1} \quad
  Changliu Liu\textsuperscript{1} \\[4pt]
  \textsuperscript{1}Robotics Institute \qquad
  \textsuperscript{2}Mechanical Engineering \\
  Carnegie Mellon University, United States \\
  \texttt{\{ruic3, shangtal, yifansu, cliu6\}@andrew.cmu.edu}
}

\begin{document}
\maketitle

% \vspace{-27pt}

\begin{figure}[h]
    \centering
    \includegraphics[width=0.8\linewidth]{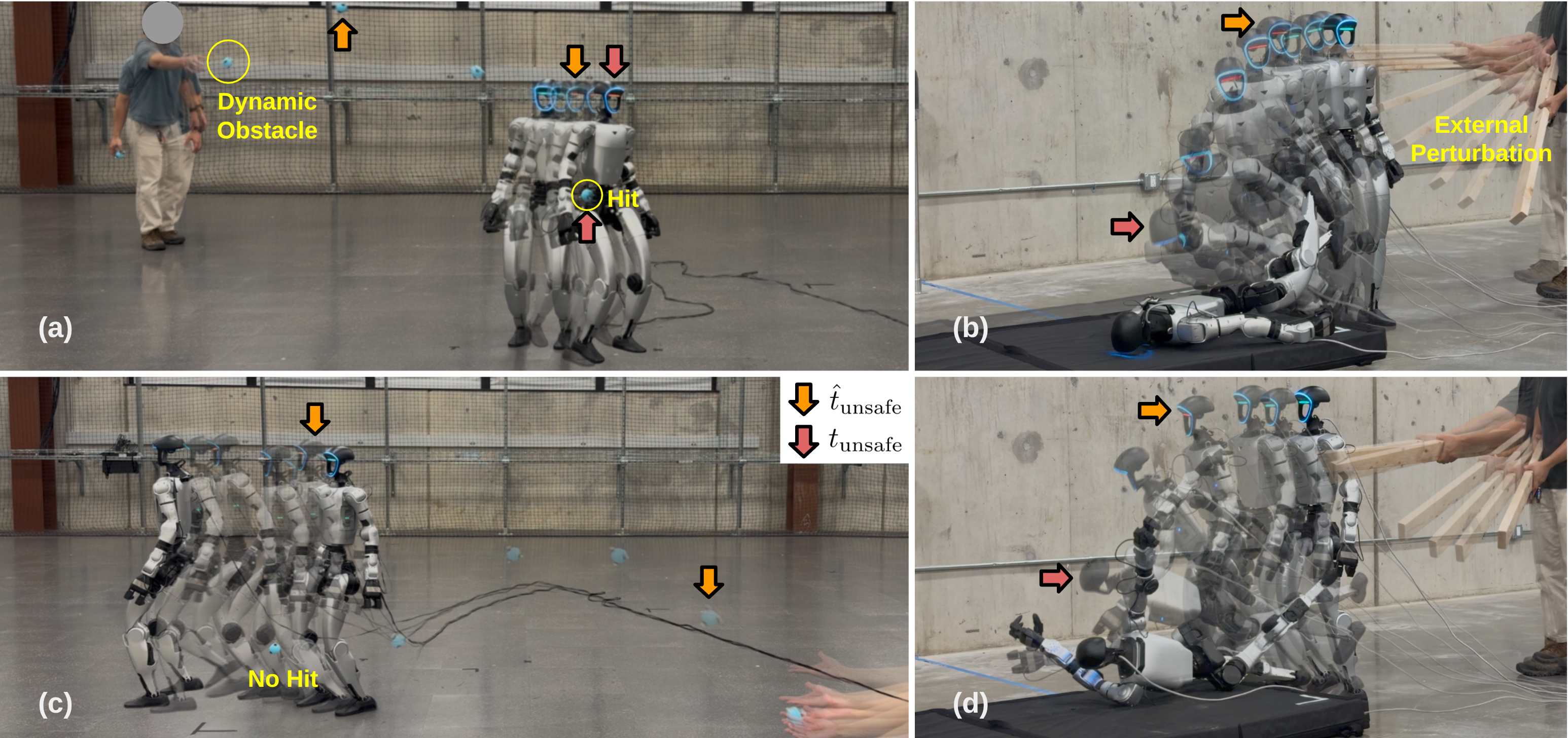}
    \caption{Safety value inference on hardware. A safety value $V^\pi$ is learned to predict the state-wise worst-over-future safety signal $\ell$, which models collision avoidance in tasks (a,c) and stable locomotion in tasks (b,d). The frame $\hat{t}_\mathrm{unsafe}$ when $V^\pi$ predicts the first non-recoverable state (orange) and frame $t_\mathrm{unsafe}$ when the actual unsafe condition occurs (red) are both marked. The gap between $\hat{t}_\mathrm{unsafe}$ and $t_\mathrm{unsafe}$ is the predictive horizon of $V^\pi$ for unsafe states. The headlight turns red at $\hat{t}_\mathrm{unsafe}$ when $V^\pi$ predicts future safety violations.}
    \label{fig:teaser}
\end{figure}

% \vspace{-12pt}

\begin{abstract}
We introduce \emph{$\lambda$-Reachability}, a scalable approach to Hamilton--Jacobi safety analysis for high-dimensional robotic systems.
Unlike prior discounted formulations that rely on fixed one-step Bellman updates, $\lambda$-Reachability employs a \emph{stochastic multi-step estimator} of the safety value, using a geometrically distributed rollout horizon together with a randomly absorbed terminal.
Conceptually analogous to TD($\lambda$), $\lambda$-Reachability interpolates between local self-consistency updates and long-horizon max-over-trajectory safety targets via an interpretable horizon-control parameter.
Unlike TD($\lambda$), where the terminal value is always incorporated in learning targets, the terminal safety value in $\lambda$-Reachability is only used with a probability controlled by parameter $\delta$.
We formally show that for $\delta<1$, the update induces a contraction mapping that allows temporal-difference learning; as $\lambda \to 1$, the estimator recovers the undiscounted reachability objective.
We apply $\lambda$-Reachability to high-dimensional safety learning problems with both simulated and real humanoid robots under balance and collision avoidance constraints.
Experimental results demonstrate that $\lambda$-Reachability significantly improves both safe-set boundary classification and safety margin estimation compared to single-step temporal-difference baselines.
\camrdy{
Project website: \url{https://intelligent-control-lab.github.io/lambda_reach/}.
}
\end{abstract}


%===============================================================================

% \section{Introduction}\label{sec:intro}

% \ruic{Random discussions for now -}

% In the literature, \textit{safety value functions} have been proposed to either identify feasible states from which the system can stay safe \citep{ganai2024hamilton}, or derive control constraints that push the system inside safety boundaries \citep{liu2014control,ames2014control}.
% The synthesis of such value functions, however, is in general intractable beyond low-dimensional tasks \citep{chen2023sis,chen2023sia,fisac2019bridging}.

% The balance between bias and variance has traditionally been addressed in return-maximizing problems using TD($\lambda$), which computes a weighted sum of value targets across varying bootstrap steps.
% Unfortunately, the same technique does not directly apply to our case, since we are computing a maximum target rather than an accumulative one.

\section{Introduction}

Robots operating in the real world must reason about \emph{persistent} safety: whether a policy will ever violate constraints, and how close it is to doing so.
A common formalization uses an instantaneous safety signal, with its positivity indicating a violation, and defines a policy's safety value as the worst future signal.
Hence, the safety value certifies safety for all future time, with the sign indicating whether persistent safety is reachable and the magnitude acting as a safety margin.
This worst-over-time notion is widely used as the foundational quantity for safety assessment \citep{ganai2024hamilton}.

% \changliu{Based on safety assessment, }
Based on the safety assessment, a wide range of work focuses on \emph{safe control}: producing actions that comply with safety constraints.
Safe set algorithms (SSA) \citep{liu2014control,liu2022safe} and control barrier functions (CBF) \citep{ames2019cbf,wang2023high} provide an optimization-based framework to enforce forward invariance within a safe set while retaining task performance.
Hamilton-Jacobi (HJ) reachability-based methods \citep{fisac2019generalSafety,ganai2024hamilton} recover the safety value by solving a variational inequality via dynamic programming.
% and have recently been extended to interactive scenarios such as human-robot collaboration \citep{pandya2025slide}. \changliu{no need to say this has been extended to HRC, since SSA and CBF has already been widely used in HRC. Instead, we can say all these approaches have been applied to solve xx, xx, xx problems (list real applications)}
Those safe control approaches have been applied to real-world applications such as legged robots \citep{hsu2023sim, yun2024safe, sun2025spark, chen2025dexterous, sun2026learning}, collaborative manipulation \citep{pandya2025slide}, and autonomous driving \citep{leung2020infusing}.
To ultimately achieve safe control, the correctness of the safety value---both the sign indicating safety and the magnitude encoding safety margin---is an essential prerequisite, as it directly quantifies the safety performance.
% \changliu{It needs a stronger reason to say "Hence, ... for fixed policies..." And also the methods reviewed in the following paragraph are not all for fixed policies. It may be better to not mention fixed policies here, but mention in the paragraph when we introduce $\lambda$-Reachability. }
Only with accurate safety values, one is able to derive appropriate safe control operations, such as activating failure-recovery policies based on safety margins \citep{he2024agile} and applying safe control constraints based on value gradients \citep{chen2023sia,yun2024safe}.
Hence, this work focuses on acquiring accurate safety values for policies as a critical step in safety analysis and control.

Prior work that explicitly \emph{synthesizes} or \emph{learns} the underlying safety value with formal structure is often limited to relatively low-dimensional state descriptions.
Safety Index Synthesis constructs a safety index by solving sum-of-squares (SOS) programs \citep{chen2023sis,chen2023sia,zhao2022sis}, but SOS formulations rely on semi-definite optimization whose size grows rapidly with state dimension and polynomial degree, leading to poor scalability \citep{ahmadi2019dsos}.
Neural CBF learning methods aim to fit barrier functions from data \citep{lindemann2021hybridcbf}, but unlike SOS- or PDE-based constructions whose constraints hold by design, learned neural barrier functions must satisfy barrier inequalities over a continuous state space; since training typically enforces these conditions only on sampled states, additional verification is required to establish safety guarantees \citep{vertovec2025lbp,zhang2024seev}.
Similarly, HJ reachability computes a viscosity solution of the HJ (variational inequality) PDEs on a grid.
However, standard solvers suffer from the curse of dimensionality---memory and computation scale exponentially with state dimension \citep{bansal2017hjreachability,darbon2016curse}.
These limitations are calling for new methods to learn safety value functions that can operate directly in high-dimensional robotic and perceptual representations.

To address the issue of high dimensionality, several works reformulate safety problems through \emph{discounted} safety Bellman operators that admit contraction properties.
This enables the use of reinforcement learning techniques, such as TD/Q-learning with function approximation \citep{fisac2019bridging,hsu2021reachavoid}, to learn safety values from data.
\citet{ganai2023iterative} learn reachability estimators to support safe policy optimization, while \citet{ganai2024hamilton} summarize the growing connections between HJ reachability and reinforcement learning.
Notably, all existing works build on the safety Bellman equation with a one-step bootstrapped value target. 
In high-dimensional settings with rich perception (e.g., height maps) and complex environment dynamics, one-step bootstrapping can be unstable since approximation errors can be propagated through noisy transitions.
The value learning can also be slow, since future safety information propagates by only one step per update.
% leading to overly conservative predictions with misclassified safe states and inaccurate safety margins.
% \changliu{justify why these are required}
% To correctly apply safety-enforcing control constraints at appropriate timings, \emph{both} accurate feasible-boundary identification and accurate margin estimation are required.
To avoid these drawbacks, one may consider removing bootstrapping entirely and learning from Monte Carlo-style targets instead.
However, that requires full task trajectories, which can be costly to collect in high-dimensional real-world environments.
It is also desirable to learn safety values while the system is in operation, instead of waiting for full rollouts which defers learning.
In summary, we aim to efficiently learn safety value functions that accurately identify feasible boundaries and estimate safety margins, even from partial trajectory data.

% \changliu{why is it off-policy?}

Motivated by these challenges, we propose \emph{$\lambda$-Reachability}, a TD$(\lambda)$-style approach \citep{sutton1988td} for safety value learning that replaces sensitive one-step backups with a geometric mixture of multi-step max targets.
By sampling a random rollout horizon and backing up the maximum observed safety signal over that horizon, $\lambda$-Reachability interpolates smoothly between short-horizon bootstrapping and Monte Carlo max targets, while providing an explicit knob ($\lambda$) that controls the effective look-ahead horizon.
Instead of deterministically bootstrapping from the current safety value estimate, we do so only with a probability $\delta^n$ that decreases as the look-ahead horizon $n$ extends.
We formally prove that the resulting stochastic multi-step safety Bellman equation induces a contraction mapping when $\delta<1$, thereby allowing temporal-difference solutions and recovering the undiscounted safety value in expectation as $\lambda\to 1$.
We also interpret our approach as a connection between CBF and HJ-based approaches in terms of safety value learning.
Finally, $\lambda$-Reachability is applied to high-dimensional robotic tasks using perceptual observations to recover safety values for balance and collision-avoidance constraints.
Results in both simulations and on physical hardware show improved boundary and margin accuracy compared to one-step baselines.

% To proceed, \Cref{sec:problem} introduces the safety value learning task as well as the discounted safety Bellman equation as first viable approach to high-dimensional problems.
% \Cref{sec:method} describes $\lambda$-Reachability and presents theoretical analysis and interpretations.
% \Cref{sec:experiment} verifies our approach against baselines and ablations on various value learning tasks.
% We cover limitations in \Cref{sec:limitation} and conclude this paper in \Cref{sec:conclusion}, focusing on the implications of $\lambda$-Reachability on the application of a broader range of RL techniques to safety problems.

\section{Background}\label{sec:problem}

% \subsection{Dynamics}

% We consider a discrete-time dynamic system.
% Let $x\in\cX\subset \RR^{N_x}$ be the system state
% and $u\in\cU$ be the control.
% The control space $\cU$ is bounded element-wise.
% % i.e., $\cU\defeq\{u\in\RR^{N_u} \mid \underline{u} \leq u \leq \bar{u} \}$.
% The dynamics is given by $x_{t+1} = f(x_t, u_t), ~ u_t \in \cU$
% % \begin{equation}\label{eq:dynamics}
% %     x_{t+1} = f(x_t, u_t), ~ u_t \in \cU,
% % \end{equation}
% where $f: \RR^{N_x\times N_u} \mapsto \RR^{N_x}$ is locally Lipschitz continuous.

\subsection{Safety Value Functions for Control Policies}\label{subsec:learn_safety_value}

Let $x\in\cX$ be the system state and $u\in\cU$ the control.
During operation, the system should never enter unsafe states such as collisions and off-balance poses.
% \changliu{this is quite humanoid specific, do you want to mention humanoids in the beginning?}.
We assume that such a constraint is specified through a \emph{safety signal}
$\ell:\cX\rightarrow\RR$, which defines a safe set we wish the system to stay in $\cX_S \defeq \{x \in \cX \mid \ell(x) \leq 0\}.$
% \begin{equation}
% \cX_S \defeq \{x \in \cX \mid \ell(x) \leq 0\}.
% \end{equation}
Given a control policy $\pi:\cX\rightarrow\cU$, let $\tau^{\pi}_x$ denote the infinite-horizon state trajectory induced by $\pi$ starting from $x$.
If $\ell(\tau^{\pi}_x(t))\le 0$ for all $t$, $\tau^{\pi}_x$ is safe and
the policy induces a forward invariant set \citep{chen2023sis} within $\cX_S$.
% the system is said to be \textit{forward invariant} \citep{chen2023sis} within $\cX_S$.
% \changliu{better to say: the policy induces a forward invariant set within $\cX_S$.}
This work aims to synthesize a safety value function $V^\pi:\cX\rightarrow\RR$ that quantifies the future safety performance of $\pi$ for any $x\in\cX$, defined as
\begin{equation}\label{eq:safety_value}
V^\pi(x_t) \defeq \sup_{k \ge t}\ell(x_k)
\end{equation}
% It follows that the system is forward invariant if $V^\pi(x_t) \le 0$.
It follows that $\{x \mid V^\pi(x) \le 0\}$ is a forward invariant set within $\cX_S$.
% \changliu{better to say: $V^\pi(x_t) \le 0$ is the forward invariant set within $\cX_S$.}
By definition, \eqref{eq:safety_value} also satisfies the \textit{safety Bellman equation}:
\begin{equation}\label{eq:undiscounted_safety_bellman}
V^\pi(x)=\mathrm{max}\{\ell(x), V^\pi(f(x,\pi(x)))\}
\end{equation}

% \begin{equation}\label{eq:value_rl}
% V \defeq \textstyle\sum_{k \ge t} \gamma^{k-t}r(x_t,u_t)
% \end{equation}

% \begin{equation}\label{eq:bellman_rl}
% V^\pi(x_t) = r(x_t,u_t) + \gamma V^\pi(x_{t+1})
% \end{equation}

\subsection{The Discounted Safety Bellman Equation}\label{subsec:discounted_bellman}

Exact synthesis of the value function \eqref{eq:safety_value} has proven intractable for higher-dimensional problems either by sum-of-squares programming \citep{zhao2022sis,chen2023sis} or solving a Hamilton-
Jacobi-Bellman variational inequality \citep{ganai2024hamilton}.
Unlike standard reward-based Bellman equations, \eqref{eq:undiscounted_safety_bellman}
does not induce a contraction mapping, which makes it unsuitable to derive a learning target directly \citep{fisac2019bridging}.
To address this issue, \citet{fisac2019bridging} introduce a $\gamma$-discounted formulation of the safety objective.
For a discount factor $\gamma\in(0,1)$,
% \changliu{avoid symbol overload here. Use $\gamma$ for the discount. Or are we making these two equal?}
the \emph{discounted safety Bellman equation} is defined as
\begin{equation}\label{eq:discounted_safety_bellman}
V^\pi(x)=(1-\gamma)\ell(x)+\gamma\mathrm{max}\{\ell(x), V^\pi(f(x,\pi(x)))\}
\end{equation}
where $\gamma$ can be interpreted as an episode continuing probability.
This equation induces a contraction mapping under the supremum norm and admits a unique fixed point.
Moreover, as $\gamma \rightarrow 1$, the solution $V^\pi$ to \eqref{eq:discounted_safety_bellman} recovers the original undiscounted safety value in \eqref{eq:safety_value}.
Concretely, the learning is carried out by regressing $V^\pi(x_t)$ towards sampled targets $y_t$:
\begin{equation}\label{eq:dpe_update}
    V^\pi_{k+1}(x_t) = V^\pi_{k}(x_t) + \alpha \left[ y_t - V^\pi_k(x_t) \right]
\end{equation}
where $k$ is the training step and $\alpha$ is the learning rate.
% \changliu{also define the learning rate $\alpha$}
The target $y_t$ is given by
\begin{equation}\label{eq:dpe_target}
    y_t = (1-\gamma)\ell_t+\gamma\mathrm{max}\{\ell_t, V^\pi_{k}(x_{t+1})\},
\end{equation}
with $\gamma$ annealed to $1$ during training.
Hence, \eqref{eq:discounted_safety_bellman} is suitable for being used with temporal-difference techniques to provide a tractable way to synthesize safety value functions \citep{fisac2019bridging}.

\section{Stochastic Multi-step Safety Bellman Equation}
\label{sec:method}

\subsection{Unifying Max Backups with Variable Horizon}

Notice that the original safety objective in \eqref{eq:safety_value} is defined over an infinite horizon, whereas the update rule in \eqref{eq:dpe_update} relies on a one-step bootstrap target at the successor state $x_{t+1}$. 
As a result, the regression target $y_t$ is strongly influenced by the current approximation of $V^\pi(x_{t+1})$, which is typically inaccurate during early stages of training. 
This induces substantial bias in the learning signal and can lead to error propagation through repeated bootstrapping. 
Moreover, such myopic updates are inefficient at propagating safety information backward in time when unsafe events occur many steps after $t$. 
Similar limitations of one-step updates are well documented in the reinforcement learning literature for classic one-step temporal-difference (TD) methods \citep{sutton1988td}.

At the opposite end of the spectrum, Monte-Carlo methods avoid bootstrapping altogether by constructing targets directly from fully observed trajectories. 
In the context of safety value learning, this corresponds to the target $y_t = \max\{\ell_t, \ell_{t+1}, \ldots, \ell_T\}$
% \begin{equation}\label{eq:mc_target}
%     y_t = \max\{\ell_t, \ell_{t+1}, \ldots, \ell_T\},
% \end{equation}
where $T$ denotes the terminal time step. 
Monte-Carlo targets are unbiased with respect to the infinite-horizon objective \eqref{eq:safety_value} and provide accurate supervision when full trajectories are available. 
However, they typically exhibit high variance, require complete rollouts, and cannot always be computed until termination. 
These properties make Monte-Carlo supervision impractical for high-dimensional systems and incompatible with online learning settings, where only partial trajectory segments may be accessible.

Motivated by the complementary strengths and weaknesses of one-step bootstrapping and Monte-Carlo evaluation, our core contribution is to unify learning from max targets with different bootstrap horizons for safety value functions. 
Specifically, we propose a variant of the safety Bellman equation with a geometric stopping horizon and stochastic terminal absorption.
The equation smoothly interpolates between short-horizon bootstrapping and long-horizon Monte-Carlo supervision. 
Specifically, given state $x_t$, a continuation probability $\lambda\in[0,1)$ and a terminal survival constant $\delta\in[0,1)$, we define the following \textit{stochastic multi-step safety Bellman equation}
\begin{align}
    V^\pi(x_t) = \EE_{n,s|x_t}\Big[ & \max\big\{ \ell_t,\ell_{t+1},\dots,\ell_{t+n-1}, b_t^{(n)} \big\} \Big] \label{eq:nstep_bellman}
\end{align}
where the stochastic bootstrapped value is given by
\begin{equation}
b_t^{(n)}\defeq s V^\pi(x_{t+n}) + (1-s) v_\mathrm{term}.
\end{equation}
The stopping step $n\ge 1$ follows a geometric distribution $n \sim \mathrm{Geom}(1-\lambda)$ given by
\begin{equation}\label{eq:geometric_dist}
    \mathbb{P}(n=k) = (1-\lambda)\lambda^{k-1}, \quad k \in \mathbb{N}.
\end{equation}

% $n \sim \mathrm{Geom}(1-\lambda)$

The terminal survival event $s \sim \mathrm{Bernoulli}(\delta^n)$ bootstraps on the current value estimation with a probability that decreases as the horizon grows.
$v_\mathrm{term}$ is a lower bound of $\ell$ so that when $s=0$, the bootstrapped value at $t+n$ is effectively ignored.
% \changliu{Not necessarily need to be a lower bound. Need to provide more discussion on the choice of $v_\mathrm{term}$.}
Importantly, although \eqref{eq:nstep_bellman} does not explicitly include a discount factor, it induces a contraction mapping which allows temporal-difference approaches to converge to a fixed point.
As $\lambda\to 1$, \eqref{eq:nstep_bellman} also converges to the undiscounted objective \eqref{eq:safety_value}.
To show those, we first define the stochastic multi-step safety Bellman operator
\begin{equation}\label{eq:lambd_operator}
T_{\lambda,\delta}[V](x_t)\defeq\EE_{n,s|x_t}\Big[ \max\big\{\ell_t,\dots,\ell_{t+n-1}, b_t^{(n)} \big\} \Big].   
\end{equation}
We then present the theoretical results as follows.
\camrdy{
See Appendix \ref{append:proof_contraction} and \ref{append:proof_value_approx} for proofs.
}

% For a control policy $\pi$, given a safety signal $\ell$ and the safety value $V^\pi$ \eqref{eq:safety_value}, 

\begin{theorem}[Contraction Mapping]\label{thm:contraction}
    The stochastic multi-step safety Bellman equation induces a contraction mapping under the supremum norm for $\lambda\in[0,1)$ and $\delta\in[0,1)$. Namely, for $V,\tilde{V}:\cX\to\RR$, there exists a constant $\rho\in[0,1)$ such that $\|T_{\lambda,\delta}[V]-T_{\lambda,\delta}[\tilde{V}]\|_\infty \le \rho\|V-\tilde{V}\|_\infty$.
\end{theorem}

\begin{proposition}[Value Approximation]\label{prop:value_approx}
    In the limit of $\lambda\to 1$, the fixed-point solution to \eqref{eq:nstep_bellman} converges to the undiscounted safety value \eqref{eq:safety_value}.
\end{proposition}

Hence, \eqref{eq:nstep_bellman} has a unique fixed-point (by \Cref{thm:contraction}) which solves the objective \eqref{eq:safety_value} as $\lambda\to 1$ (by \Cref{prop:value_approx}).
These properties justify the learning of $V^\pi$ by regression towards
\begin{equation}\label{eq:nstep_target}
y_t^{(n)} \defeq \max\left\{
\ell_t,\ell_{t+1},\dots,\ell_{t+n-1}, b_t^{(n)}
\right\}
\end{equation}
which is a stochastic multi-step estimator that integrates safety signals observed along the trajectory segment while retaining the ability to bootstrap from the current value approximation at the stopping state. 
We refer to this learning procedure as \emph{$\lambda$-Reachability}, in analogy to TD($\lambda$) methods for cumulative rewards.
% \eqref{eq:nstep_target} is an unbiased estimator of the true objective, while avoiding the instability associated with deterministic one-step bootstrapping.
Notably, the bootstrap horizon $n$ above is unbounded, whereas in practice we only have access to finite-length policy rollouts.
Hence, for practical implementation, we instead sample $n$ from a truncated geometric distribution.
See Appendix \ref{append:lambda_training} for full training implementations.

\vspace{-5pt}

% \ruic{REVISE}

% \begin{theorem}[Unbiased Estimation]\label{thm:equiv_discounted}
%     For a control policy $\pi$, given a safety signal $\ell$ and the safety value $V^\pi$ \eqref{eq:safety_value}, the target $y_t^{(n)}$ \eqref{eq:nstep_target} with a geometric horizon $n\sim\mathrm{Geom}(1-\lambda)$ for $\lambda\in[0,1]$ satisfies:
%     \begin{equation}
%         \EE[{y_t^{(n)}}] = (1-\lambda)\ell(x_t)+\lambda\max\{\ell(x_t), \EE[{y_{t+1}^{(n)}}]\}
%     \end{equation}
% \end{theorem}
% \begin{proof}
%     xxx
% \end{proof}

% As a result, \eqref{eq:nstep_target} is an unbiased stochastic estimator of the corresponding $\lambda$-discounted safety value, while avoiding the instability associated with deterministic one-step bootstrapping. 

% \changliu{We can use this formulation to build a connection between CBF and HJ. }

% \ruic{More comments on unlocking a family of RL techniques with smooth bias-variance trade-offs}

\subsection{Interpretations}\label{sec:interpretations}

% \subsubsection{Relation to Discounted Safety Bellman Equation}

\textbf{Explicit Safety Horizon.}
The expected stopping length is $\mathbb{E}[n] = \frac{1}{1-\lambda}$
% \begin{equation}\label{eq:expected_n}
% \mathbb{E}[n] = \frac{1}{1-\lambda}
% \end{equation}
which increases monotonically as $\lambda \to 1$. 
$\lambda$ provides an explicit mechanism for tuning the safety horizon: smaller values emphasize near-term safety, while larger values aggressively propagate long-horizon safety information. 
This control can be particularly useful when different notions of safety prediction are needed.

% \ruic{shrink this and add analysis of source of contraction: DPE from explicit value discount, lambda from random inclusion of terminal bootstrap in expection}

\textbf{Relation to Discounted Safety Bellman.}
Two properties must hold to allow temporal-difference learning: (i) a unique fixed point induced by a contraction for training stability and (ii) consistency with the true safety value for value approximation accuracy.
While both \eqref{eq:discounted_safety_bellman} and \eqref{eq:nstep_bellman} satisfy these conditions, they differ fundamentally in how contraction is induced.
The discounted Bellman induces contraction via explicit value discounts \eqref{eq:discounted_safety_bellman}, requiring $\gamma < 1$.
Meanwhile, accurate value approximation is achieved in the limit $\gamma \to 1$ \citep{fisac2019bridging}.
Hence, the training stability and approximation accuracy are inherently coupled.
In contrast, $\lambda$-Reachability \eqref{eq:nstep_bellman} incorporates the bootstrapped value with probability $\delta^n$, inducing a contraction mapping as long as $\delta<1$.
% in expectation for any $\delta < 1$, while converging to the undiscounted safety objective as $\lambda \to 1$.
The value approximation accuracy is determined by a different parameter $\lambda$.
Such decoupling allows one to improve value accuracy by annealing $\lambda$ to $1$, while maintaining contraction and training stability by $\delta<1$.

% \subsubsection{Fixed-Point Solution and Value Approximation}
% Two properties must hold to allow temporal-difference learning of value functions: (a) unique fixed-point and (b) value approximation by definition.
% While these properties are satisfied by both the discounted safety Bellman equation \eqref{eq:discounted_safety_bellman} and our stochastic multi-step variant \eqref{eq:nstep_bellman}, their dependencies on hyper-parameters differ.
% \eqref{eq:discounted_safety_bellman} is contractive if $\gamma<1$, while it approximates the value function only when $\gamma\to 1$ \citep{fisac2019bridging}, making the two properties coupled.
% As $\gamma\to 1$, the value approximation becomes more accurate at the cost of weaker contraction.
% In our case, \eqref{eq:nstep_bellman} is contractive as long as $\delta < 1$ and approximates the value function accurately if $\lambda\to 1$.
% With decoupled properties, \eqref{eq:nstep_bellman} may accurately approximate the value function by annealing $\lambda$ to $1$, while keeping $\delta < 0$ to preserve the contraction property.

\textbf{Relation to CBF and HJ.}
For safety value learning, CBF–based approaches~\citep{lindemann2021hybridcbf} directly exploit invariance labels that encode infinite-horizon safety, consistent with \eqref{eq:safety_value}. 
These methods learn the persistent safety, but do not explicitly model how safety information propagates across timesteps. 
In contrast, HJ reachability approaches~\citep{ganai2024hamilton} rely on one-step self-consistency conditions \eqref{eq:undiscounted_safety_bellman} and propagate safety information through bootstrapping. 
While effective, purely local updates can be slow in high-dimensional settings.
In this regard, $\lambda$-Reachability unifies these two paradigms via observing future safety signals with a stochastic look-ahead horizon, and interpolating between global invariance supervision and local HJ consistency.
From \eqref{eq:nstep_bellman}, one can recover CBF learning by fixing $s=0$ and using trajectory-specific horizons $n$,
% \changliu{a fixed horizon (?) Finite makes it sound biased }
and recover HJ learning by fixing $s=1$ and $n=1$.

\vspace{-10pt}

% \vspace{-10pt}

\section{Experiments}
\label{sec:experiment}

\subsection{Tasks}
\label{sec:exp_task}

% The dataset for each safety value learning problem is constructed through the following three steps.

\textbf{1) Policy training.}
Given a task, we first train a policy $\pi$ using task-specific objectives (e.g., velocity tracking).
Note that a safety value can be learned for an arbitrary policy.
% A safety signal is not involved in this step yet.

\textbf{2) Policy rollout.}
Given $\pi$, we collect a set of finite-length rollouts $\{\tau^\pi\}$ from the safety value learning environment.
 % is independent from policy training, but it is one in which a safety signal $\ell(x)$ is evaluated.
% $\ell(x) > 0$ indicates safety violation.
% To ensure meaningful coverage of both safe and unsafe scenarios, we initiate each episode with a random event that actively perturbs the robot towards unsafe regions.
To collect unsafe scenarios, we sample a random event that perturbs the robot towards unsafe regions at the start of each episode.
% The event will not repeat during the episode.
If the episode ends in a safe state, the robot is assumed to remain safe with a sufficiently trained policy.
% to preserve safety in the absence of perturbations.
% \changliu{this is based on the assumption that the policy is well trained.}
This enables us to label infinite-horizon safety metrics (i.e., signs and safety margins) from finite trajectories for evaluation.
In practice, such an assumption may not hold, making safety labels empirical estimates that can be biased.

\textbf{3) Safety labeling.}
For each state $x_t \in \tau^\pi$, we first compute the safety signal $\ell(x)$.
Then, we define the forward-invariance indicator $c_t \triangleq \mathbf{1}\{\ell(x_k) \le 0,\ \forall k \ge t\}$.
% \begin{equation}
% c_t \triangleq \mathbf{1}\{\ell(x_k) \le 0,\ \forall k \ge t\},
% \end{equation}
% which indicates whether safety is maintained for the remainder of the rollout.
Since rollouts are finite, $c_t$ serves as the best available estimate of forward invariance, rather than the true infinite-horizon label.
Second, we define the empirical safety value $\bar{V}^{\pi}(x_t) \triangleq \max_{k \ge t} \ell(x_k)$.
% \begin{equation}
% \bar{V}^{\pi}(x_t) \triangleq \max_{k \ge t} \ell(x_k),
% \end{equation}
% which corresponds to the worst observed safety signal along the rollout.
The resulting dataset is $\mathcal{D} \triangleq \{(x_t, \ell_t, c_t, \bar{V}^{\pi}(x_t))\}_{t}$.
% \begin{equation}
% \mathcal{D} \triangleq \{(x_t, \ell_t, c_t, \bar{V}^{\pi}(x_t))\}_{t}.
% \end{equation}
Tasks used in this work are defined below.
See \Cref{fig:envs} for visualizations.

% \begin{itemize}
%     \item
\textbf{\textsc{G1Flat}}.
We train a locomotion policy $\pi$ for a Unitree G1 humanoid robot with 37 DoF in a flat environment with velocity-tracking goals.
In this task, we consider a safety signal $\ell_\mathrm{flat}$ to indicate the balance status, defined by both root height and body orientation.
% \begin{equation}

% \item 
\textbf{\textsc{G1Rough}}.
This task replaces the flat ground in G1Flat with rough terrains with slopes and stairs with scan points as additional input.
Other settings remain the same.

% \item
\textbf{\textsc{G1Collision}}.
This task builds on \textsc{G1Flat} and introduces dynamic obstacles that the robot must avoid while navigating.
To calculate the safety signal $\ell_\mathrm{collision}$, we introduce a distance component $\ell_\mathrm{dist}=\min\{\max\{(d_\mathrm{min}-d)/d_\mathrm{min}, -1\}, 0\}$ and a contact component $\ell_\mathrm{contact}=1$ if the ball hits the robot and $-1$ otherwise.
Hence, we have $\ell_\mathrm{collision} = \max\{\ell_\mathrm{flat}, \ell_\mathrm{dist}, \ell_\mathrm{contact}\}$.

See Appendix \ref{append:training_details} for full environment setup, policy training, and data collection details.
% \begin{equation}\label{eq:safety_signal_flat_avoid}
%     \ell_\mathrm{collision} = \max\{\ell_\mathrm{flat}, \ell_\mathrm{dist}, \ell_\mathrm{contact}\}
% \end{equation}

% \end{itemize}

% See \Cref{fig:envs} for the environments.
% In all tasks, the safety value $V^\pi$ becomes positive (top ball turns red) before actual safety violations ($\ell>0$ and lower ball turns red), as it captures the future safety performance.
% The goal of safety value learning is to recover such $V^\pi$ from the roll out dataset $\cD$ for the given policy $\pi$.

\begin{figure}[t]
    \centering
    \includegraphics[width=\linewidth]{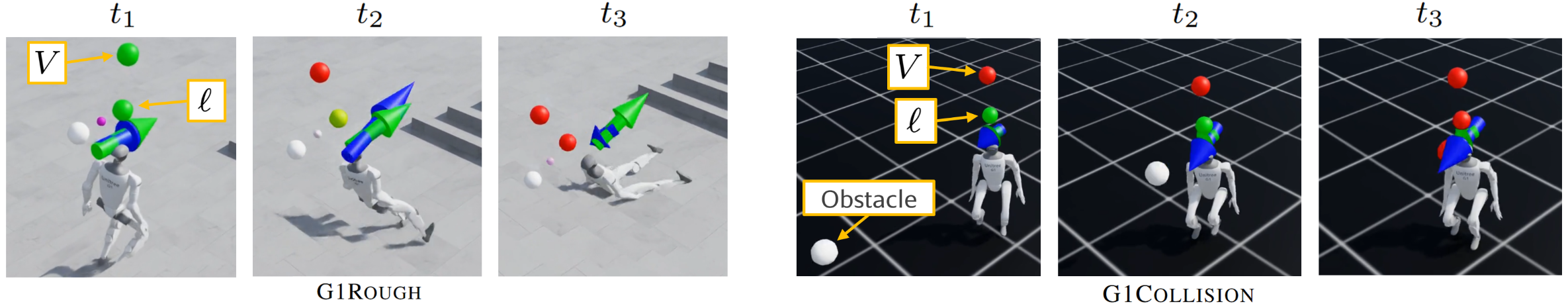}
    \caption{Humanoid safety value tasks. The safety signal $\ell$ is indicated by the ball color, with green indicating safe states ($\ell\le -0.5$), yellow indicating near violations ($\ell\in(-0.5,0]$), and red indicating safety violations ($\ell>0$).
    The predicted safety value $V$ is indicated by the higher ball with the same color coding.
    % Values in the above examples are computed by $\lambda$-Reachability.
    % \textsc{G1Flat} uses the same environment as \textsc{G1Collision}.
    }
    \vspace{-10pt}
    \label{fig:envs}
\end{figure}

\subsection{Ablations and Baselines}\label{sec:safety_value_training_main}

$\lambda$-Reachability is trained by regressing towards the geometric-horizon targets \eqref{eq:nstep_target} by minimizing
% $\cL_\mathrm{main} = \EE_{x_t\sim\tau^\pi, n\sim\mathrm{Geom(1-\lambda)}}\|V^\pi(x_t) - y_t^{(n)}\|^2$
\begin{equation}\label{eq:lambda_loss_overall}
    \cL_\mathrm{main} = \EE_{x_t\sim\tau^\pi, n\sim\mathrm{Geom(1-\lambda)}}\|V^\pi(x_t) - y_t^{(n)}\|^2
\end{equation}
% We minimize \eqref{eq:lambda_loss_overall}
via gradient descent with a learning rate of $\alpha=1e-3$ for $2000$ steps.
$\lambda$ is set to $0.99$.
To examine the effect of $\lambda$ value, we deploy ablations by setting alternative $\lambda$ values in $\{0.95, 0.5, 0.0\}$.
\camrdy{
We also perform ablation on $n_\mathrm{max}$ in $\{50,100,150,200\}$ and $\delta$ in $\{0.5, 0.9, 0.95, 0.99\}$.
}
% Auxiliary losses are scaled and added to the main loss during training.
We also compare $\lambda$ to other baselines that recover the safety value \eqref{eq:safety_value}.

\textbf{Discounted Policy Evaluation (DPE)}.
DPE \citep{fisac2019bridging} recovers $V^\pi$ by regressing towards the discounted safety Bellman target \eqref{eq:dpe_target}.
The discount factor $\gamma$ is initialized to $0.9$ and annealed to $0.99$ during training.
This serves as the main baseline which $\lambda$-Reachability builds on.

% \textbf{Weakly Supervised (WS)}.
% WS trains on the estimated invariance label $c_t$ only by minimizing the binary cross-entropy loss.

\textbf{Supervised (SV)}.
SV fits the estimated safety value $\bar{V}^{\pi}(x_t)$.
This represents the optimum in highly controlled tasks.
% described in \Cref{sec:exp_task}.
In general, it is difficult to identify ``events'' and segment trajectories into finite episodes that capture infinite-horizon safety.
Labels from partial trajectories would be biased.
Hence, SV is considered a miracle baseline for evaluation only, rather than a practical method.

% \begin{itemize}
%     \item 1-step 
%     \item Weakly supervised
%     \item $\lambda$ with all losses, $\lambda = 0.99$
%     \item $\lambda$ without partial loss
%     \item $\lambda$ with $\lambda = 0.9$
%     \item $\lambda$ with $\lambda = 0.0$ (expected 1-step bootstrap)
%     \item Supervised
% \end{itemize}

\begin{figure*}[t]
    \centering

    % --- Row 1 (currently 2 plots, future-proof to 4) ---
    \begin{subfigure}[t]{0.31\textwidth}
        \centering
        \includegraphics[width=\linewidth]{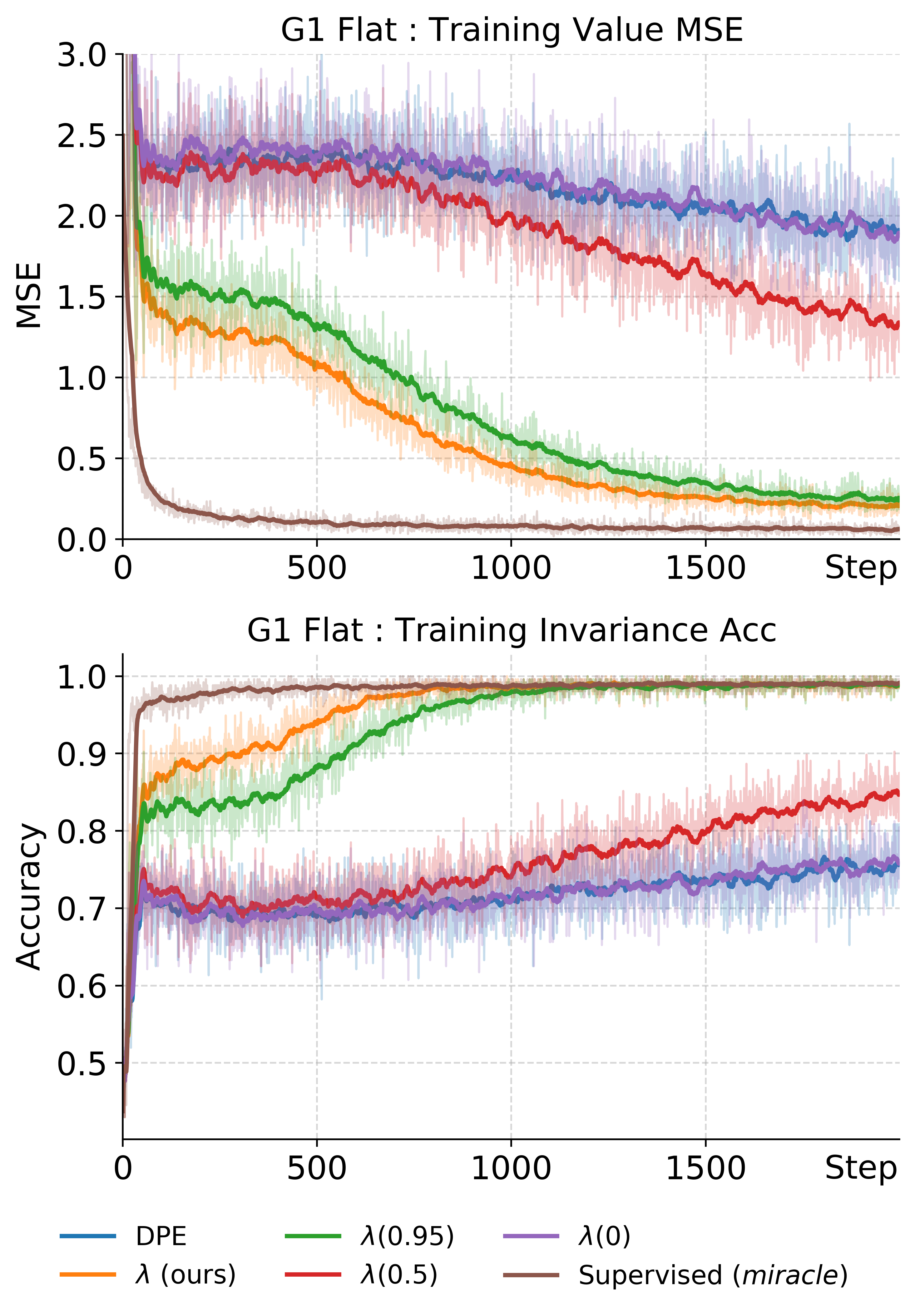}
        \caption{G1 Flat}
        \label{fig:g1_flat_train}
    \end{subfigure}
    \hfill
    \begin{subfigure}[t]{0.31\textwidth}
        \centering
        \includegraphics[width=\linewidth]{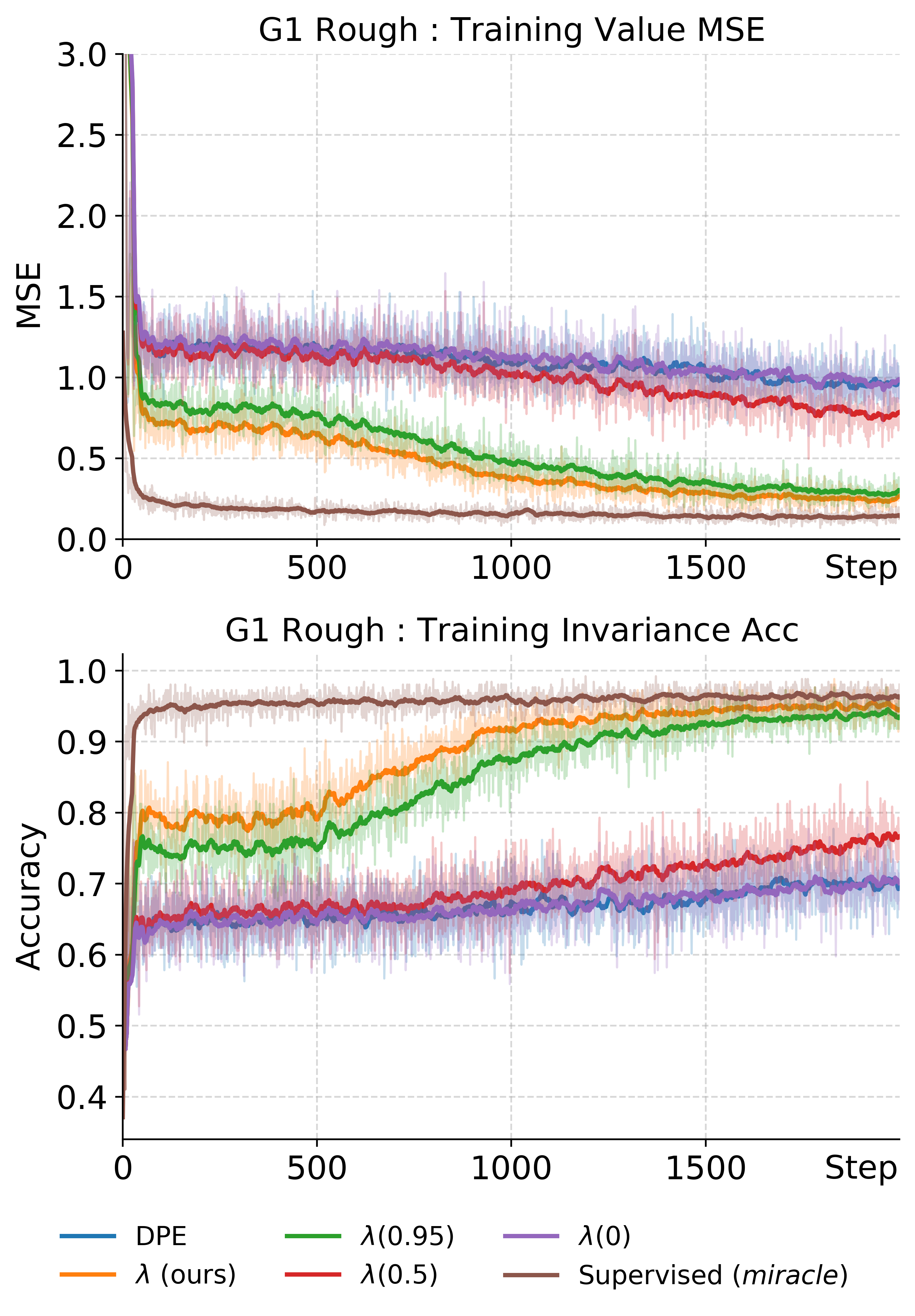}
        \caption{G1 Rough}
        \label{fig:g1_rough_train}
    \end{subfigure}
    \hfill
    \begin{subfigure}[t]{0.31\textwidth}
        \centering
        \includegraphics[width=\linewidth]{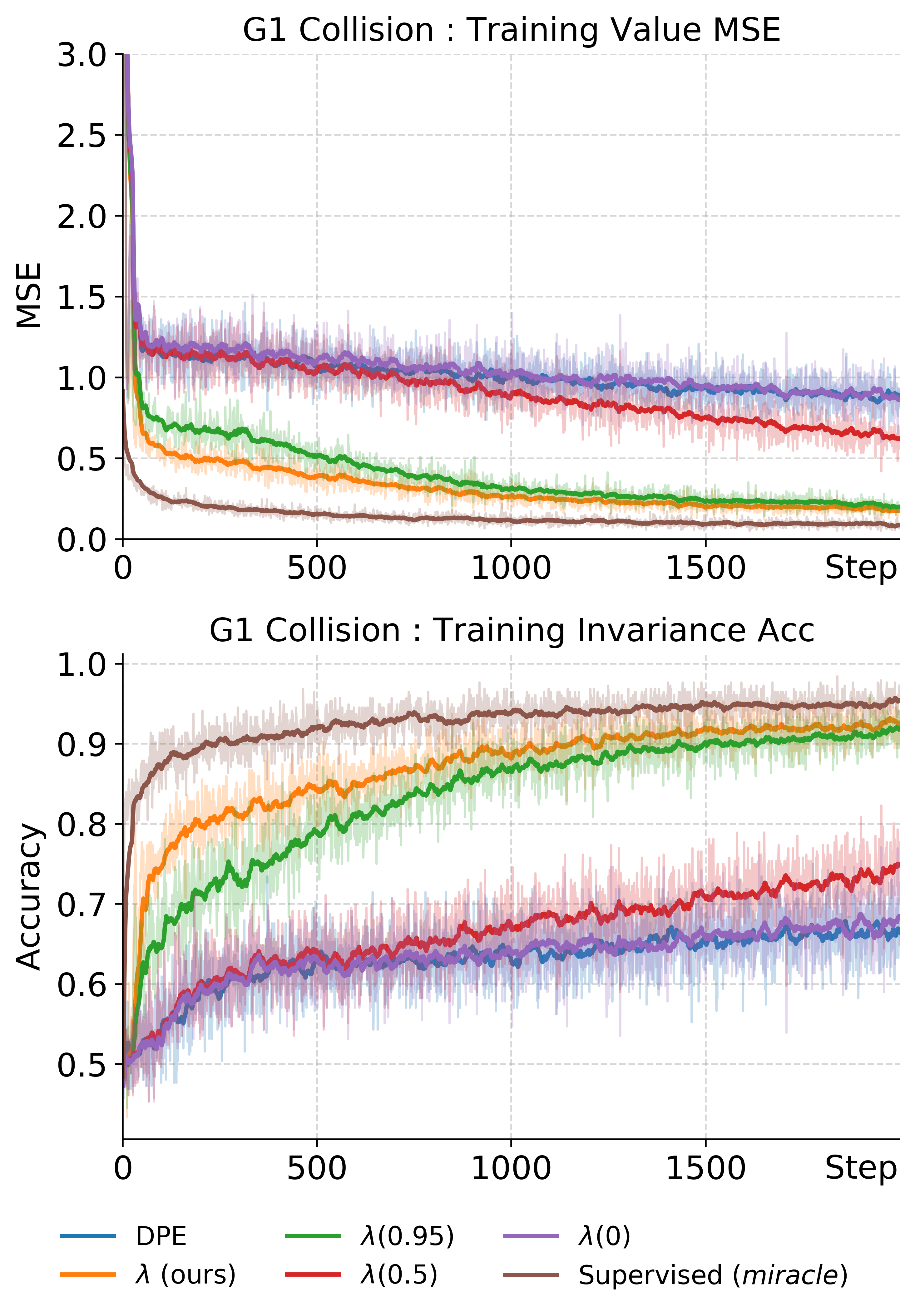}
        \caption{G1 Collision}
    \end{subfigure}
    % \hfill
    % \begin{subfigure}[t]{0.24\textwidth}
    %     \centering
    %     \includegraphics[width=\linewidth]{<fourth_plot>.png}
    %     \caption{<Fourth Caption>}
    % \end{subfigure}

    \caption{
        Training curves for safety value functions in different safety tasks.
        The top row shows the sample value error $e_V$ during training, while the bottom row shows the invariant state classification accuracy. 
    }
    \label{fig:training_curves}
    \vspace{-10pt}
\end{figure*}

\subsection{Metrics}\label{sec:safety_value_inference_metrics}

For each state, the safety value function $V^\pi$ should correctly reveal both (a) whether the system will stay safe and (b) the current safety margin.
When applied as a safety monitor, $V^\pi$ should warn of future unsafe states as early as possible.
That leads to the following metrics.
% Hence, we evaluate each safety value learning approach with the following metrics.

\textbf{Temporal Recall $r_\mathrm{temp}$} is the ratio of the predictive horizon for unsafe states over the maximal possible horizon: $r_\mathrm{temp} = \EE_{\tau^\pi}[\max(0,t_\mathrm{unsafe}-\hat{t}_\mathrm{unsafe})/(t_\mathrm{unsafe}-t_0)]$
% \begin{equation}
%     r_\mathrm{temp} = \EE_{\tau^\pi}[\max(0,t_\mathrm{unsafe}-\hat{t}_\mathrm{unsafe})/t_\mathrm{unsafe}]
% \end{equation}
where $t_\mathrm{unsafe}$, $\hat{t}_\mathrm{unsafe}$ are the earliest timestep at which $\ell$, $V^\pi(x_t) \geq 0$ respectively.
% The episode starts at $t=0$.
$r_\mathrm{temp}=1$ indicates the optimal value.

\textbf{Sample Value Error $e_\mathrm{V}$} is the mean squared error $e_\mathrm{V} = \EE_{x_t\sim\tau^\pi}|V^\pi(x_t)-\bar{V}^{\pi}(x_t)|^2$.
% \begin{equation}
%     e_\mathrm{V} = \EE_{x_t\sim\tau^\pi}|V^\pi(x_t)-\bar{V}^{\pi}(x_t)|^2
% \end{equation}

\textbf{Sample False Positive Rate $r_\mathrm{FPR}$} indicates the chances that $V^\pi$ misclassifies a non-invariant state $r_\mathrm{FPR} = p(V^\pi(x_t)\le 0 ~|~ c_t=0)$.
% \begin{equation}
%     r_\mathrm{FPR} = p(V^\pi(x_t)\le 0 ~|~ c_t=0)
% \end{equation}

% Metrics:
% \begin{itemize}
%     \item Temporal Recall
%     \item Invariance Classification Accuracy
%     \item Safety Value Error
% \end{itemize}

\begin{table}[t]
\centering
\caption{Inference results (mean $\pm$ std) in simulation. \textcolor{miracleolive}{Supervised} is marked as the miracle baseline. Best values among non-miracle methods are shown in bold.}
\label{tab:quantitative}
\renewcommand{\arraystretch}{1.15}
\setlength{\tabcolsep}{3.0pt}
\scriptsize
\resizebox{\linewidth}{!}{%
\begin{tabular}{lccc|ccc|ccc}
\toprule
 & \multicolumn{3}{c|}{\textsc{G1Flat}} 
 & \multicolumn{3}{c|}{\textsc{G1Rough}} 
 & \multicolumn{3}{c}{\textsc{G1Collision}} \\
\cmidrule(lr){2-4} \cmidrule(lr){5-7} \cmidrule(lr){8-10}
\textbf{Method} 
& {$r_\mathrm{temp}$ (\%)} 
& {$e_\mathrm{V}$} 
& {$r_\mathrm{FPR}$ (\%)}
& {$r_\mathrm{temp}$ (\%)} 
& {$e_\mathrm{V}$} 
& {$r_\mathrm{FPR}$ (\%)}
& {$r_\mathrm{temp}$ (\%)} 
& {$e_\mathrm{V}$} 
& {$r_\mathrm{FPR}$ (\%)} \\
\midrule

$\lambda$ (ours) 
& \textbf{99.98 $\pm$ 0.21} & \textbf{0.09 $\pm$ 0.19} & \textbf{0.21}
& \textbf{97.92 $\pm$ 9.24} & \textbf{0.08 $\pm$ 0.22} & \textbf{6.49}
& \textbf{99.40 $\pm$ 4.79} & 0.13 $\pm$ 0.23 & \textbf{6.12} \\

$\lambda~(0.95)$
& 99.97 $\pm$ 0.29 & 0.11 $\pm$ 0.23 & 0.38
& 96.43 $\pm$ 11.34 & 0.09 $\pm$ 0.25 & 9.23
& 98.63 $\pm$ 5.96 & \textbf{0.12 $\pm$ 0.20} & 9.15 \\

$\lambda~(0.5)$
& 56.52 $\pm$ 20.39 & 0.71 $\pm$ 1.23 & 31.91
& 52.40 $\pm$ 27.65 & 0.27 $\pm$ 0.64 & 45.45
& 44.80 $\pm$ 18.28 & 0.16 $\pm$ 0.33 & 48.02 \\

$\lambda~(0.0)$
& 26.39 $\pm$ 14.09 & 1.00 $\pm$ 1.65 & 47.42
& 21.97 $\pm$ 19.06 & 0.39 $\pm$ 0.88 & 60.17
& 23.85 $\pm$ 14.53 & 0.18 $\pm$ 0.38 & 59.63 \\

$\mathrm{DPE}$ 
& 22.05 $\pm$ 12.30 & 1.04 $\pm$ 1.70 & 49.31
& 21.94 $\pm$ 16.66 & 0.40 $\pm$ 0.89 & 59.68
& 22.61 $\pm$ 14.50 & 0.18 $\pm$ 0.37 & 60.47 \\

\rowcolor{miracleolive!12}
\miracle{$\mathrm{Supervised}$} 
& \miracle{99.99 $\pm$ 0.33} & \miracle{0.05 $\pm$ 0.13} & \miracle{0.09}
& \miracle{99.35 $\pm$ 6.29} & \miracle{0.09 $\pm$ 0.17} & \miracle{2.71}
& \miracle{99.72 $\pm$ 2.60} & \miracle{0.14 $\pm$ 0.27} & \miracle{2.05} \\

\bottomrule
\end{tabular}
}
\vspace{-10pt}
\end{table}

\subsection{Quantitative Results in Simulation}\label{sec:sim_quantitatives}

% \ruic{see appendix for environment (dof model, obs comp), safety value learning (lr, delta, terminal value, Batch, n\_max, tau, etc.) details}

We summarize the training results in \Cref{fig:training_curves} and present evaluations in \Cref{tab:quantitative}.
See Appendix \ref{append:training_details} for implementation details.
Notably, \textsc{G1Flat} is already a challenging safety value learning problem, with substantially higher state dimensionality than prior benchmarks \citep{chen2023sis,ganai2024hamilton}.
\textsc{G1Rough} further increases the complexity through height-map perception, while \textsc{G1Collision} is challenging in safety reasoning with multiple constraints.
% including balance and dynamic obstacle avoidance.
% Together, these tasks capture representative challenges encountered in realistic environments, making the results a meaningful proxy for evaluating the practical performance of different methods.

\textbf{Baselines}.
From \Cref{fig:training_curves}, we see that $\lambda$-Reachability quickly reduces the safety value prediction error and improves the invariance classification accuracy.
The performance converges to that of the miracle baseline, showing that $\lambda$-Reachability is able to learn a near-optimal safety value even with partial trajectories.
On the other hand, discounted policy evaluation (DPE) has trouble in improving both metrics.
Since DPE relies on fixed one-step updates \eqref{eq:discounted_safety_bellman}, the propagation of unsafe events can be slow and cause inefficient learning in high-dimensional tasks.
The above gap is also pronounced in \Cref{tab:quantitative} where $\lambda$-Reachability shows a significant advantage over DPE in early warning of impending safety violations, accurate safety margins and rare false positives.
With those, $\lambda$-Reachability validates as an efficient and reliable approach to high-dimensional safety value learning.

\textbf{Ablations}.
From both \Cref{fig:training_curves} and \Cref{tab:quantitative}, we see that the choice of $\lambda$ has a major impact on the performance.
With $\lambda=0.99$, $\lambda$-Reachability achieves near-optimal safety learning across all tasks, matching the performance of the miracle baseline.
As $\lambda$ decreases, the expected bootstrap horizon drops according to \Cref{sec:interpretations}, e.g., $20$ steps at $\lambda=0.95$ and $2$ steps at $\lambda=0.5$.
That quickly degrades the performance until it matches the performance of DPE at $\lambda=0$.
This agrees with the fact that with $\lambda=0$, the horizon sampling \eqref{eq:geometric_dist} becomes deterministic and always samples $1$, effectively reducing the target \eqref{eq:nstep_target} to the one-step baseline \eqref{eq:dpe_target} with $\gamma=1$.
Hence, we confirm that local bootstrapping cannot reliably propagate future safety violations, especially when safety-relevant events occur far from the current state.
\camrdy{
From \Cref{fig:ablation-training} and \Cref{tab:rebuttal-combined-ablation}, we reveal that $\lambda$ is not sensitive to $n_\mathrm{max}$, while $\delta$ can slow down the convergence and degrade especially the classification of non-invariant states ($r_\mathrm{FPR}$).
}

\begin{figure}[ht]
  \centering

  \begin{subfigure}{\linewidth}
    \centering
    \includegraphics[width=0.8\linewidth]{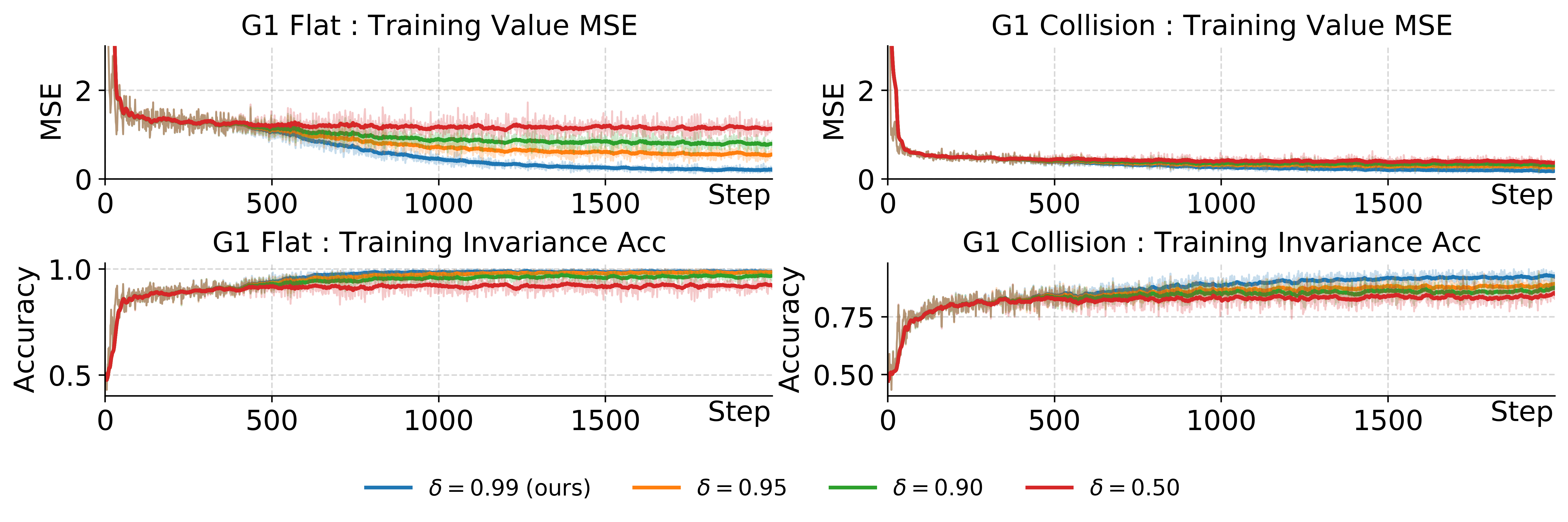}
    % \caption{Training curves with varying $\delta$.}
    \label{fig:delta_ablation_training}
  \end{subfigure}

  % \vspace{-5pt}

  \begin{subfigure}{\linewidth}
    \centering
    \includegraphics[width=0.8\linewidth]{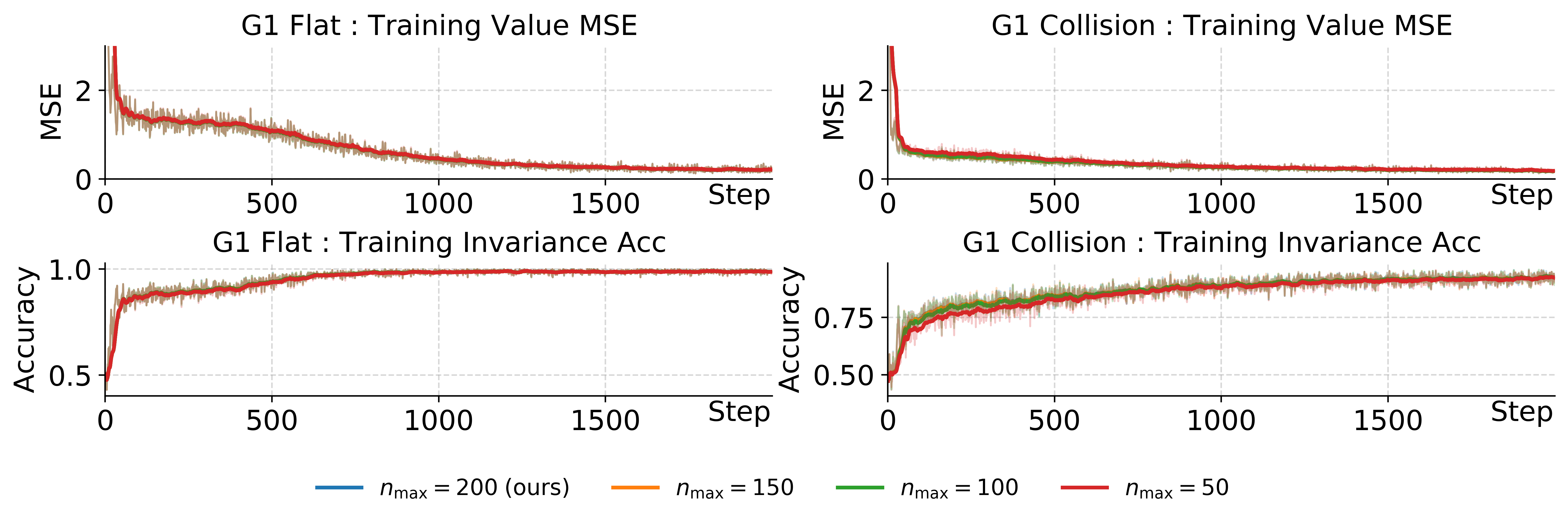}
    % \caption{Training curves with varying $n_\mathrm{max}$.}
    \label{fig:nmax_ablation_training}
  \end{subfigure}

  % \vspace{-5pt}

  \caption{Training curves for ablating $\delta$ and $n_\mathrm{max}$.}
  \label{fig:ablation-training}

  % \vspace{-5pt}
  
\end{figure}

% \input{corl_rebuttal/fig/delta_ablation_table}

% \input{corl_rebuttal/fig/nmax_ablation_table}

% Auto-generated by safety_value/scripts/export_rebuttal_subset_tables.py.
% Requires \usepackage{booktabs,graphicx}.
\begin{table}[t]
  \centering
  \caption{Sensitivity to $\delta$ and $n_{\max}$. Best values are bold.}
  \label{tab:rebuttal-combined-ablation}
  \renewcommand{\arraystretch}{1.15}
  \setlength{\tabcolsep}{3.0pt}
  \scriptsize
  % \resizebox{\linewidth}{!}{%
  \begin{tabular}{llccc}
    \toprule
    \textbf{Task} & \textbf{Method} & {$r_\mathrm{temp}$ (\%)} & {$e_\mathrm{V}$} & {$r_\mathrm{FPR}$ (\%)} \\
    \midrule
    % \multicolumn{5}{c}{\textbf{$\delta$ ablation}} \\
    \textsc{G1Flat} & $\delta~(0.99)$ (ours) & \textbf{99.99 $\pm$ 0.17} & \textbf{0.09 $\pm$ 0.19} & \textbf{0.21} \\
    \textsc{G1Flat} & $\delta~(0.9)$ & 94.56 $\pm$ 9.68 & 0.43 $\pm$ 0.76 & 7.76 \\
    \textsc{G1Flat} & $\delta~(0.5)$ & 82.47 $\pm$ 17.47 & 0.62 $\pm$ 1.05 & 17.54 \\
    \midrule
    \textsc{G1Collision} & $\delta~(0.99)$ (ours) & \textbf{99.36 $\pm$ 4.87} & 0.13 $\pm$ 0.23 & \textbf{6.19} \\
    \textsc{G1Collision} & $\delta~(0.9)$ & 93.93 $\pm$ 17.26 & \textbf{0.11 $\pm$ 0.23} & 18.12 \\
    \textsc{G1Collision} & $\delta~(0.5)$ & 91.70 $\pm$ 21.70 & 0.12 $\pm$ 0.25 & 22.23 \\
    \midrule
    % \multicolumn{5}{c}{\textbf{$n_{\max}$ ablation}} \\
    \textsc{G1Flat} & $n_{\max}~(200)$ (ours) & \textbf{99.99 $\pm$ 0.17} & 0.09 $\pm$ 0.19 & 0.21 \\
    \textsc{G1Flat} & $n_{\max}~(100)$ & 99.99 $\pm$ 0.18 & \textbf{0.09 $\pm$ 0.19} & \textbf{0.20} \\
    \textsc{G1Flat} & $n_{\max}~(50)$ & 99.98 $\pm$ 0.21 & 0.10 $\pm$ 0.19 & 0.22 \\
    \midrule
    \textsc{G1Collision} & $n_{\max}~(200)$ (ours) & \textbf{99.36 $\pm$ 4.87} & 0.13 $\pm$ 0.23 & \textbf{6.19} \\
    \textsc{G1Collision} & $n_{\max}~(100)$ & 99.29 $\pm$ 5.09 & 0.13 $\pm$ 0.22 & 6.21 \\
    \textsc{G1Collision} & $n_{\max}~(50)$ & 98.92 $\pm$ 5.63 & \textbf{0.12 $\pm$ 0.20} & 7.49 \\
    \bottomrule
  \end{tabular}%
  % }
  % \vspace{-10pt}
\end{table}

\subsection{Hardware Experiments}\label{sec:hardware_quantitatives}

\begin{table}[t]
  \centering
  \caption{Hardware inference results (mean $\pm$ std) on real \textsc{G1Flat} and \textsc{G1Collision} tasks.
  % \textcolor{miracleolive}{Supervised} is marked as the miracle baseline. Best values among non-miracle methods are shown in bold.
  }
  \label{tab:quantitative_hw}
  \renewcommand{\arraystretch}{1.15}
  \setlength{\tabcolsep}{3.0pt}
  \scriptsize
  \resizebox{0.7\linewidth}{!}{%
  \begin{tabular}{lccc|ccc}
  \toprule
   & \multicolumn{3}{c|}{\textsc{G1Flat}}
   & \multicolumn{3}{c}{\textsc{G1Collision}} \\
  \cmidrule(lr){2-4} \cmidrule(lr){5-7}
  \textbf{Method}
  & {$r_\mathrm{temp}$ (\%)}
  & {$e_\mathrm{V}$}
  & {$r_\mathrm{FPR}$ (\%)}
  & {$r_\mathrm{temp}$ (\%)}
  & {$e_\mathrm{V}$}
  & {$r_\mathrm{FPR}$ (\%)} \\
  \midrule

  $\lambda$ (ours)
  & 80.91 $\pm$ 4.96 & \textbf{0.19 $\pm$ 0.15} & 20.31
  & \textbf{73.59 $\pm$ 9.20} & 0.48 $\pm$ 0.25 & \textbf{30.54} \\

  $\lambda~(0.95)$
  & 82.24 $\pm$ 10.15 & 0.21 $\pm$ 0.16 & 20.31
  & 72.09 $\pm$ 9.13 & \textbf{0.47 $\pm$ 0.25} & 31.83 \\

  $\lambda~(0.5)$
  & \textbf{84.26 $\pm$ 11.78} & 0.24 $\pm$ 0.14 & \textbf{18.77}
  & 48.78 $\pm$ 27.09 & 0.58 $\pm$ 0.33 & 48.35 \\

  $\lambda~(0.0)$
  & 74.25 $\pm$ 22.35 & 0.26 $\pm$ 0.19 & 31.80
  & 41.29 $\pm$ 23.66 & 0.66 $\pm$ 0.39 & 53.99 \\

  $\mathrm{DPE}$
  & 62.07 $\pm$ 19.42 & 0.44 $\pm$ 0.32 & 45.21
  & 13.52 $\pm$ 6.97 & 0.75 $\pm$ 0.52 & 72.20 \\

  \rowcolor{miracleolive!12}
  \miracle{$\mathrm{Supervised}$}
  & \miracle{84.42 $\pm$ 13.26} & \miracle{0.27 $\pm$ 0.26} & \miracle{19.92}
  & \miracle{79.63 $\pm$ 9.67} & \miracle{0.48 $\pm$ 0.33} & \miracle{23.21} \\

  \bottomrule
  \end{tabular}
  }
  % \vspace{-10pt}
  \end{table}

We perform \textsc{G1Flat} and \textsc{G1Collision} tasks on a real Unitree G1 humanoid.
% Safety values $V^\pi$ are learned for a robot policy $\pi$ in simulation with necessary adaptations for real-world deployment.
% $V^\pi$ is then applied to rollout data collected from physical experiments.
$V^\pi$ is learned in simulation and applied to physical rollout data.
To perturb the safety signal, we push the robot in \textsc{G1Flat} and throw balls at the robot in \textsc{G1Collision} where the robot and ball are tracked using a mocap system.
We summarize results in \Cref{tab:quantitative_hw}.
$\lambda$-Reachability stably outperforms the DPE baseline, and matches the performance of the miracle baseline with proper hyperparameters.
Results on the \textsc{G1Collision} task show trends similar to those in the previous section.
% where a larger $\lambda$ leads to stronger predictive performance for unsafe conditions, more accurate estimation of the safety margin, and fewer misclassifications of unrecoverable states.
For \textsc{G1Flat}, however, $\lambda$-Reachability with high $\lambda$ values shows less advantage, with top scores at $\lambda=0.5$.
One potential cause is that $V^\pi$ in \textsc{G1Flat} relies entirely on noisy, real proprioceptive signals.
$\lambda$ trained in simulation with longer bootstrap horizons (i.e., larger $\lambda$) relies on predictable future trajectories, which are rare in practice.
$\lambda$ with shorter horizons leans more on the near-term signal and is more robust under distribution shift.
In \textsc{G1Collision}, although the same noises exist, the safety values largely depend on the obstacle, whose trajectories are smooth and predictable thanks to the mocap system.
\camrdy{
In general, the average horizon $1/(1-\lambda)$ should reflect common time-to-failure periods.
}
See Appendix \ref{append:training_details} for training details and Appendix \ref{append:additional_hw_results} for additional results.
% \ruic{hw diffs: robot model, obs space, ppo policy asym actor critic, actor not seeing base vel, etc.}

% \ruic{
% \begin{itemize}
%     \item perform g1 flat, g1 collision task in real setup. real g1 dof model. push trigger. mocap, ball.
%     \item repeat previous policy training, data rollout, and training of baseline, ablation in sim.
%     \item take result $V^\pi$ and apply to hardware data
% \end{itemize}
% }

% \ruic{screen shot vs signal/value plot}

% \ruic{quantitatives}
\section{Limitations}
\label{sec:limitation}

We identify two aspects where future work would benefit.
First, safety value learning requires infinite-horizon trajectories to be sampled, since the horizon is essentially unbounded (see \eqref{eq:geometric_dist}).
In practice, however, policy rollouts are finite with truncated horizon sampling (see \eqref{eq:truncated_geometric_dist}).
Hence, practical $\lambda$ learning can be biased.
It is thus beneficial to derive a max-horizon-dependent term as future work to quantify the learning bias with truncated bootstrap horizons.
Another aspect regards the tradeoff between bias and robustness.
As shown previously, noisy real data can introduce distributional shifts that degrade the performance of long-horizon $\lambda$ values.
On the other hand, the solution to \eqref{eq:nstep_bellman} is unbiased only when $\lambda$ anneals to $1$.
Consequently, $\lambda$ values closer to $1$ reduce bias while making the learned value less robust against noisy online input.
A principled approach to the bias-robustness tradeoff would be beneficial for applying $\lambda$ in general conditions.

\section{Discussions and Conclusions}
\label{sec:conclusion}

In this paper, we proposed $\lambda$-Reachability, a safety value learning approach that solves the safety Bellman equation using a geometric mixture of multi-step max backups.
$\lambda$-Reachability induces both a contraction mapping and unbiased safety value approximation under mild conditions, enabling scalable and efficient temporal-difference learning approaches.
Experiments across multiple value learning problems demonstrate that $\lambda$-Reachability consistently outperforms the baseline, addressing the issue of slow information propagation inherent to one-step bootstrapping in high-dimensional settings.
Although this work focuses on policy evaluation and does not derive a control policy, it can be directly integrated into existing reinforcement learning frameworks to enable safety-aware learning.
For instance, $\lambda$-Reachability can replace value learning in actor--critic methods (e.g., TD3, SAC), yielding actor updates that explicitly maximize safety.
This improves upon prior actor--critic-based safety approaches such as \citep{pandya2025slide}, in which value functions need only be correct up to a partial order.
Crucially, safety-maximizing actor-critics rely on accurate value scales to determine the feasibility of actions and will therefore benefit from the accurate safety margins of $\lambda$.
Similarly, $\lambda$ can be used in policy optimization (e.g., TRPO, PPO) by replacing value targets with geometric-horizon max targets~\eqref{eq:nstep_target}, similar to~\citep{fisac2019bridging}.
Hence, while $\lambda$-Reachability focuses on policy evaluation, it provides a principled approach to learning high-dimensional safety values, which is beneficial across a broad class of safe policy learning problems.

\section{Acknowledgment}

This work is partially supported by the National Science Foundation, Grant No. 2144489.

%===============================================================================

\clearpage
% The acknowledgments are automatically included only in the final and preprint versions of the paper.
% \acknowledgments{If a paper is accepted, the final camera-ready version will (and probably should) include acknowledgments. All acknowledgments go at the end of the paper, including thanks to reviewers who gave useful comments, to colleagues who contributed to the ideas, and to funding agencies and corporate sponsors that provided financial support.}

%===============================================================================

\bibliography{references}  % .bib

\newpage
\appendix

\section{Proof of Contraction Mapping}\label{append:proof_contraction}

% \begin{theorem}[Contraction Mapping]\label{thm:contraction_appendix}
%     The stochastic multi-step safety Bellman equation induces a contraction mapping under the supremum norm for $\lambda\in[0,1)$ and $\delta\in[0,1)$. Namely, for $V,\tilde{V}:\cX\to\RR$, there exists a constant $\rho\in[0,1)$ such that $\|T_{\lambda,\delta}[V]-T_{\lambda,\delta}[\tilde{V}]\|_\infty \le \rho\|V-\tilde{V}\|_\infty$.
% \end{theorem}

\noindent\textbf{Theorem 1} (Contraction Mapping).
\textit{The stochastic multi-step safety Bellman equation induces a contraction mapping under the supremum norm for $\lambda\in[0,1)$ and $\delta\in[0,1)$. Namely, for $V,\tilde{V}:\cX\to\RR$, there exists a constant $\rho\in[0,1)$ such that $\|T_{\lambda,\delta}[V]-T_{\lambda,\delta}[\tilde{V}]\|_\infty \le \rho\|V-\tilde{V}\|_\infty$.}

\begin{proof}
    % Define the operator \changliu{this operator should be defined outside the proof and better to be before the Theorem is introduced.}
    For brevity, let $m_t^{(n)}\defeq\max_{t\le k \le t+n-1}\ell_k$.
    Let $V_t\equiv V(x_t)$ and expand the terminal sampling, we have
    \begin{align}
        T_{\lambda,\delta}[V](x_t) = \EE_{n|x_t}\big[&\delta^n\max(m_t^{(n)}, V_{t+n}) \nonumber\\
        & + (1-\delta^n)\max(m_t^{(n)},v_\mathrm{term})\big].
    \end{align}
    Then, $\forall x_t\in\cX$,
    \begin{align}
        & \left|T_{\lambda,\delta}[V](x_t)-T_{\lambda,\delta}[\tilde{V}](x_t)\right| \nonumber \\
        = & \left|\EE_{n|x_t}\Big[ \delta^n(\max(m_t^{(n)}, V_{t+n})-\max(m_t^{(n)}, \tilde{V}_{t+n})) \Big]\right| \nonumber \\
        \le & \EE_{n|x_t}\Big[ \delta^n \left| \max(m_t^{(n)}, V_{t+n})-\max(m_t^{(n)}, \tilde{V}_{t+n}) \right| \Big] \nonumber \\
        \le & \EE_{n|x_t}\Big[ \delta^n \left| V_{t+n} - \tilde{V}_{t+n} \right| \Big] \nonumber \\
        \le & \EE_{n|x_t}[ \delta^n ] \sup_{x} \left| V(x) - \tilde{V}(x) \right| \nonumber \\
        = & \EE_{n|x_t}[ \delta^n ] \| V - \tilde{V} \|_\infty. \nonumber
    \end{align}
    Hence, we have $\|T_{\lambda,\delta}[V]-T_{\lambda,\delta}[\tilde{V}]\|_\infty \le \EE_{n|x_t}[ \delta^n ] \| V - \tilde{V} \|_\infty$.
    The contraction factor is indeed $\rho=\EE_{n|x_t}[ \delta^n ]\in[0,1)$.
\end{proof}

\newpage
\section{Proof of Value Approximation}\label{append:proof_value_approx}
\noindent\textbf{Proposition 1} (Value Approximation).
\textit{In the limit of $\lambda \to 1$, the fixed-point solution to \eqref{eq:nstep_bellman} converges to the undiscounted safety value \eqref{eq:safety_value}.}

\begin{proof}
Fix any $\delta \in [0,1)$, and let $V_{\lambda,\delta}$ denote the
unique fixed point of $\mathcal{T}_{\lambda,\delta}$, whose existence
follows from Theorem~1. For brevity, let
\[
m_t^{(n)}
:=
\max_{t \leq k \leq t+n-1} \ell_k,
\qquad
V_t^\pi
:=
V^\pi(x_t)
=
\sup_{k \geq t} \ell_k,
\]
where $\ell_k := \ell(x_k)$. By separating the first $n$ steps from the
remaining trajectory, for every $n \geq 1$,
\[
V_t^\pi
=
\max\left\{
m_t^{(n)}, V_{t+n}^\pi
\right\}.
\]
Moreover, since $v_{\mathrm{term}}$ is a lower bound of $\ell$,
\[
\max\left\{
m_t^{(n)},v_{\mathrm{term}}
\right\}
=
m_t^{(n)}.
\]

The operator $\mathcal{T}_{\lambda,\delta}$ is monotone: if
$V \leq \widetilde V$ pointwise, then
$\mathcal{T}_{\lambda,\delta}[V]
\leq
\mathcal{T}_{\lambda,\delta}[\widetilde V]$.
Expanding the terminal sampling and applying the identities above gives
\begin{align*}
\mathcal{T}_{\lambda,\delta}[V^\pi](x_t)
&=
\mathbb{E}_{n \mid x_t}
\left[
\delta^n
\max\left\{
m_t^{(n)},V_{t+n}^\pi
\right\}
+
(1-\delta^n)
\max\left\{
m_t^{(n)},v_{\mathrm{term}}
\right\}
\right]
\\
&=
\mathbb{E}_{n \mid x_t}
\left[
\delta^n V_t^\pi
+
(1-\delta^n)m_t^{(n)}
\right]
\\
&\leq
V_t^\pi,
\end{align*}
where the final inequality follows from
$m_t^{(n)} \leq V_t^\pi$.
Monotonicity therefore implies
\[
\mathcal{T}_{\lambda,\delta}^{\,j+1}[V^\pi] \leq \mathcal{T}_{\lambda,\delta}^{\,j}[V^\pi] \leq V^\pi.
\]
Since $\mathcal{T}_{\lambda,\delta}$ is a contraction,
its iterates converge to its unique fixed point, yielding
\[
V_{\lambda,\delta}(x_t)
=
\lim_{j\to\infty}
\mathcal{T}_{\lambda,\delta}^{\,j}[V^\pi](x_t)
\leq
V_t^\pi.
\]

On the other hand, the maximum in the stochastic target always contains
$m_t^{(n)}$. Thus, by the fixed-point equation,
\begin{align*}
V_{\lambda,\delta}(x_t)
&=
\mathcal{T}_{\lambda,\delta}
[V_{\lambda,\delta}](x_t)
\\
&\geq
\mathbb{E}_{n \mid x_t}
\left[
m_t^{(n)}
\right]
=: A_\lambda(x_t).
\end{align*}
Using $n \sim \operatorname{Geom}(1-\lambda)$, this lower bound is
\[
A_\lambda(x_t)
=
(1-\lambda)
\sum_{n=1}^{\infty}
\lambda^{n-1}m_t^{(n)}.
\]

The sequence of prefix maxima is nondecreasing and satisfies
\[
m_t^{(n)}
\uparrow
V_t^\pi
\qquad\text{as } n\to\infty.
\]
To show that $A_\lambda(x_t)$ has the same limit, let $\epsilon>0$.
There exists an integer $K$ such that
\[
m_t^{(K)}
\geq
V_t^\pi-\epsilon.
\]
Since
$\mathbb{P}(n\geq K)=\lambda^{K-1}$ and
$m_t^{(n)}\geq v_{\mathrm{term}}$ for every $n$, we obtain
\begin{align*}
A_\lambda(x_t)
&\geq
v_{\mathrm{term}}\,
\mathbb{P}(n<K)
+
m_t^{(K)}\,
\mathbb{P}(n\geq K)
\\
&=
v_{\mathrm{term}}
\left(1-\lambda^{K-1}\right)
+
m_t^{(K)}\lambda^{K-1}.
\end{align*}
Taking $\lambda \uparrow 1$ while keeping $K$ fixed gives
\[
\liminf_{\lambda\uparrow 1}
A_\lambda(x_t) \geq m_t^{(K)}
\geq
V_t^\pi-\epsilon.
\]
Meanwhile, $m_t^{(n)}\leq V_t^\pi$ for every $n$, and hence
\[
A_\lambda(x_t)\leq V_t^\pi.
\]
Since $\epsilon>0$ is arbitrary,
\[
\lim_{\lambda\uparrow 1}
A_\lambda(x_t)
=
V_t^\pi.
\]

Combining the lower and upper bounds,
\[
A_\lambda(x_t)
\leq
V_{\lambda,\delta}(x_t)
\leq
V_t^\pi,
\]
and applying the squeeze theorem yields
\[
\lim_{\lambda\uparrow 1}
V_{\lambda,\delta}(x_t)
=
V_t^\pi
=
V^\pi(x_t).
\]
Therefore, the fixed-point solution to \eqref{eq:nstep_bellman} converges pointwise to the
undiscounted safety value \eqref{eq:safety_value}. Equivalently, when the bootstrap value in
(10) is evaluated at the fixed point,
\[
\lim_{\lambda\uparrow 1}
\mathbb{E}_{n,s\mid x_t}
\left[
y_t^{(n)}
\right]
=
\lim_{\lambda\uparrow 1}
V_{\lambda,\delta}(x_t)
=
V^\pi(x_t),
\]
establishing the asymptotic unbiasedness of the estimator.
% \hfill$\square$
\end{proof}

\newpage
\section{\texorpdfstring{$\lambda$}{lambda}-Reachability Implementation}\label{append:lambda_training}

\subsection{Practical Implementations}

\textbf{Truncated Geometric Distribution.}
In practice, trajectories are finite and only a limited number of future states are available for each anchor time step.
Sampling from an untruncated geometric distribution would therefore place non-negligible probability mass on horizons beyond the available data, implicitly biasing updates toward terminal states.
To address this issue, we sample the bootstrap horizon from a truncated geometric distribution with finite support and renormalized probabilities,
\begin{equation}\label{eq:truncated_geometric_dist}
\mathbb{P}(n = k)
=
\frac{(1-\lambda)\lambda^{k-1}}{1-\lambda^{T-t}},
\quad
k \in \{1, \ldots, T-t\},
\end{equation}
where $T$ denotes the terminal time of the trajectory.
This truncation preserves the geometric structure of the stopping distribution while ensuring that all sampled horizons are feasible given the available data.
% \changliu{Need to discuss what type of bias this will introduce}

\textbf{Value Initialization.}
Safety signals $\ell(x)$ are non-positive within the safe set, and the reachability update relies on a max operator over future safety values.
As a consequence, initializing the value function with overly pessimistic (positive)
% \changliu{this should be pessimistic?}
values can lead to self-reinforcing errors, whereby spurious unsafe predictions are repeatedly propagated backward through bootstrapping.
To prevent this failure mode, we initialize the value function conservatively as
\begin{equation}
V^\pi_0(x) \le \min_x \ell(x),
\end{equation}
% \changliu{Is this being initialized with a single small number?}
which guarantees that early bootstrapped targets are dominated by observed safety signals rather than inaccurate value estimates.
This initialization bias is gradually corrected as unsafe events are encountered during training.
In practice, we assign $V^\pi_0(x)$ to a constant which is by design a lower bound of $\ell$.

% \ruic{according to results so far, no auxiliary loss here makes a difference. consider removing loss ablation;}

% \ruic{changing lambda value changes the performance a lot, as it alters the expected horizon; will keep as part of ablation}

% \ruic{slow updating of target network affects training stability, but this may not be worth showing as it's standard RL procedure.}

% \ruic{the effect of double Q is TBD}

\textbf{Stabilizing Training.}
We adopt several stabilization techniques from value-based learning to improve robustness under bootstrapped max backups.
First, we use a slowly-updated target network to compute bootstrap values in \eqref{eq:nstep_target}.
Concretely, we maintain two online critics $V^\pi_{\theta_1}, V^\pi_{\theta_2}$ and their corresponding target networks $V^-_{\bar\theta_1}, V^-_{\bar\theta_2}$.
Following the clipped target idea in double-Q methods \citep{fujimoto2018addressing}, 
% \changliu{cite}
we compute the bootstrap term conservatively as
\begin{equation}\label{eq:clipped_bootstrap}
V^-_{\mathrm{boot}}(x) \defeq \min\{V^-_{\bar\theta_1}(x),\,V^-_{\bar\theta_2}(x)\},
\end{equation}
and use $V^-_{\mathrm{boot}}(x_{t+n})$ in place of $V_\pi(x_{t+n})$ in \eqref{eq:nstep_target}.
Both online critics regress to the same target, which substantially reduces the propagation of spurious over-estimation through the max operator.
Target networks are updated via Polyak averaging, $\bar\theta \leftarrow (1-\tau)\bar\theta + \tau\theta$.
% Appendix \ref{append:lambda_training} summarizes the full $\lambda$-Reachability training procedure.
% Full training details are reported in Appendix \ref{append:training_details}.

\newpage
\subsection{\texorpdfstring{$\lambda$}{lambda}-Reachability Training Procedures}
\begin{algorithm}[ht]
\caption{$\lambda$-Reachability Training}
\label{alg:lambda_reach}
\begin{algorithmic}[1]
\Require Dataset of trajectories $\{\tau\}$ collected under policy $\pi$
\Require Geometric parameter $\lambda\in[0,1)$; terminal condition parameter $\delta\in[0,1)$; absorption safety value $v_\mathrm{term}\le\min_x\ell(x)$; batch size $B$; max horizon $n_{\max}$; learning rate $\alpha$; Polyak rate $\tau$
\State \textbf{Initialize} two online critics $V^\pi_{\theta_1}, V^\pi_{\theta_2}$ conservatively (e.g., $V^\pi_{\theta_j}(x)\le \min_x \ell(x)$)
\State \textbf{Initialize} target critics $V^-_{\bar\theta_1}\gets V^\pi_{\theta_1}$, $V^-_{\bar\theta_2}\gets V^\pi_{\theta_2}$
\For{each training iteration}
    \State Sample a minibatch of anchor indices $\{(\tau_i,t_i)\}_{i=1}^B$ with futures available up to $n_{\max}$
    \For{each sample $i=1,\ldots,B$}
        \State Let $x_i \gets x^{\tau_i}_{t_i}$, $x_i^+ \gets x^{\tau_i}_{t_i+1}$, and $\ell_i \gets \ell(x_i)$
        \State Let $\bar n_i \gets \min(T_i-t_i,\,n_{\max})$
        \State Sample horizon $n_i \sim \mathrm{TruncGeom}(1-\lambda;\,1,\bar n_i)$, i.e.,
        \[
        \mathbb{P}(n_i=k)=\frac{(1-\lambda)\lambda^{k-1}}{1-\lambda^{\bar n_i}},
        \quad k\in\{1,\ldots,\bar n_i\}
        \]
        \State Compute conservative bootstrap
        \[
        V_i \gets \min\!\left\{V^-_{\bar\theta_1}(x^{\tau_i}_{t_i+n_i}),\,V^-_{\bar\theta_2}(x^{\tau_i}_{t_i+n_i})\right\}
        \]
        \State Sample terminal condition
        \[
            s_i \sim \mathrm{Bernoulli}(\delta^{n_i})
        \]
        \State Compute max target
        \[
            y_i \gets \max\{\ell(x^{\tau_i}_{t_i}),\ldots,\ell(x^{\tau_i}_{t_i+n_i-1}),\, s_i V_i + (1-s_i) v_\mathrm{term}\}
        \]
        \State Evaluate online critics:
        \[
        v_{j,i} \gets V^\pi_{\theta_j}(x_i),
        \quad j\in\{1,2\}
        \]
    \EndFor
    \State Compute loss (mean over batch):
    \[
    \mathcal{L} \gets \frac{1}{B}\sum_{i=1}^B\sum_{j=1}^2 (v_{j,i}-y_i)^2
    \]
    \State Gradient step:
    \[
    \theta_j \gets \theta_j - \alpha \nabla_{\theta_j}\mathcal{L}, \quad j\in\{1,2\}
    \]
    \State Update target critics (Polyak averaging):
    \[
    \bar\theta_j \gets (1-\tau)\bar\theta_j + \tau\theta_j, \quad j\in\{1,2\}
    \]
\EndFor
\end{algorithmic}
\end{algorithm}

\newpage
\section{Implementation Details}\label{append:training_details}

We present the implementation details of the entire pipeline used to train and evaluate safety value functions for the \textsc{G1Flat}, \textsc{G1Rough}, and \textsc{G1Collision} tasks both in simulation and on physical hardware.
For each task, we keep the humanoid modeling consistent across policy training, rollout collection, safety value learning, and safety value inference.
For simulation inference, all stages are performed in simulation.
For hardware inference, all stages except the safety value inference are done in simulation.
Adaptations are made to accommodate the absence of certain hardware observations, for example, link-obstacle contact signals. 

All training and inference steps are performed on a single workstation
equipped with an NVIDIA RTX 4090 (24\,GB) GPU, running PyTorch
2.7.0+cu128, Python 3.11, IsaacLab 0.46, and \texttt{rsl\_rl}.

\subsection{Robot Policy Training}

We first describe the environment setup and training details for acquiring the robot policy $\pi$.

\subsubsection{Policy Training for Safety Value Inference in Simulation}
For simulation experiments in \Cref{sec:sim_quantitatives}, we use a Unitree G1 humanoid robot with 37 DoF.
The observation spaces for all three tasks are listed in \Cref{tab:sim_obs}.

% \ruic{insert table showing the observation state composition of g1flat, g1rought, and g1collision; each task occupies a column, and each row is a segment of the observation (show short description of what it is, no need to add symbols; the final row shows the dimension of each row; the bottom row should show the total dims; use a row even if this segment only exists for a certain task, tasks missing that segment of observation renders empty}

\begin{table}[h]
\centering
\small
\caption{Observation composition for the simulation policies. All entries are
single-frame dimensions (no history stacking). The critic shares the actor's
observation group (symmetric actor--critic).}
\label{tab:sim_obs}
\begin{tabular}{lccc}
\toprule
Observation segment & \textsc{G1Flat} & \textsc{G1Rough} & \textsc{G1Collision} \\
\midrule
Base linear velocity $\mathbf{v}\in\mathbb{R}^3$            & 3 & 3 & 3 \\
Base angular velocity $\boldsymbol{\omega}\in\mathbb{R}^3$  & 3 & 3 & 3 \\
Projected gravity $\mathbf{g}\in\mathbb{R}^3$               & 3 & 3 & 3 \\
Velocity command $\mathbf{c}\in\mathbb{R}^3$                & 3 & 3 & 3 \\
Joint positions $\mathbf{q}\in\mathbb{R}^{37}$              & 37 & 37 & 37 \\
Joint velocities $\dot{\mathbf{q}}\in\mathbb{R}^{37}$       & 37 & 37 & 37 \\
Last action $\mathbf{a}\in\mathbb{R}^{37}$                  & 37 & 37 & 37 \\
Height scan $\mathbf{s}_{\rm scan}\in\mathbb{R}^{187}$ ($17{\times}11$ grid, $0.1$\,m) & --- & 187 & --- \\
Ball position in base frame $\mathbf{p}_{\rm ball}\in\mathbb{R}^3$ & --- & --- & 3 \\
Ball velocity in base frame $\mathbf{v}_{\rm ball}\in\mathbb{R}^3$ & --- & --- & 3 \\
\midrule
Total observation dimension & 123 & 310 & 129 \\
\bottomrule
\end{tabular}
\end{table}

In \textsc{G1Flat} and \textsc{G1Rough}, the safety signal is defined as
\begin{equation}
    \ell_\mathrm{flat}\defeq \max\left\{\ell_\mathrm{root}, \ell_\mathrm{ori}\right\},~\textrm{where}~\ell_\mathrm{root}=-\frac{h_\mathrm{root}-h_\mathrm{min}}{h^*_\mathrm{root}-h_\mathrm{min}},~\ell_\mathrm{ori} = \frac{\psi_g-\psi_\mathrm{max}}{\psi_\mathrm{max}}.
\end{equation}
The robot is considered to be safe and balanced if (a) its root height $h_\mathrm{root}$ is higher than the threshold $h_\mathrm{min}=0.2m$ and (b) the deviation $\psi_g$ between the root Z axis and gravity is at most $\psi_\mathrm{max}=\pi/4$.
$h^*_\mathrm{root}=0.65m$ is the default standing root height.

In \textsc{G1Collision}, a lightweight ball ($50g$, $10cm$ radius) is spawned every $3$ to $5$ seconds and launched toward the humanoid robot.
The ball spawns at a random position at $[2m,5m]$ from the robot, at heights between $[0.1m,1.2m]$, and is given an initial velocity of $[3m/s,6m/s]$ aimed at the robot's torso at $0.8m$ height with gravity compensation to follow a ballistic arc.

Rewards used for training $\pi$ are summarized in \Cref{tab:sim_rewards}.

% \ruic{same format as above. show each reward component (show formula and follow paper convention) for each task, and have a weight column reporting the weights of those components}

% Replaces: \ruic{same format as above. show each reward component (show formula
%   and follow paper convention) for each task, and have a weight column ...}
% Location: Sec. 3.1, immediately after "Rewards used for training $\pi$ are
%   summarized below."
% Requires: \usepackage{makecell} (and booktabs, which the prior version also uses).
\begin{table}[ht]
\centering
\small
\setlength{\tabcolsep}{4pt}
\renewcommand{\arraystretch}{1.15}
\caption{Reward composition for the simulation PPO policies. A dash (---)
means the term is not used for that task. For \textsc{G1Collision} the
ball-related terms are curriculum-gated, becoming active only after iteration
$500$; episodes terminate at $20$\,s or on illegal \texttt{torso\_link}
contact (${>}1$\,N), with ball impacts above $0.4$\,m excluded from
termination.}
\label{tab:sim_rewards}
\begin{tabular}{@{}l l c c c@{}}
\toprule
Term & Formula & \textsc{G1Flat} & \textsc{G1Rough} & \textsc{G1Collision} \\
\midrule
Lin.\ vel.\ tracking
  & \makecell[l]{$\exp(-\lVert\mathbf{v}_{xy}-\mathbf{c}_{xy}\rVert^{2}/\sigma^{2})$\\ $\sigma{=}0.5$}
  & $1.0$ & $1.0$ & $1.0$ \\
Ang.\ vel.\ tracking
  & \makecell[l]{$\exp(-(\omega_{z}-c_{\omega_{z}})^{2}/\sigma^{2})$\\ $\sigma{=}0.5$}
  & $1.0$ & $2.0$ & $1.0$ \\
$z$-velocity penalty & $v_{z}^{2}$ & $-0.2$ & $0$ & $-0.2$ \\
Roll/pitch penalty & $\lVert\boldsymbol{\omega}_{xy}\rVert^{2}$ & $-0.05$ & --- & $-0.05$ \\
Joint torques & $\lVert\boldsymbol{\tau}\rVert^{2}$
  & $-2{\times}10^{-6}$ & $-1.5{\times}10^{-7}$ & $-2{\times}10^{-6}$ \\
Joint accelerations & $\lVert\ddot{\mathbf{q}}\rVert^{2}$
  & $-1{\times}10^{-7}$ & $-1.25{\times}10^{-7}$ & $-1{\times}10^{-7}$ \\
Action rate & $\lVert\mathbf{a}_{t}-\mathbf{a}_{t-1}\rVert^{2}$
  & $-0.005$ & $-0.005$ & $-0.005$ \\
Feet air-time bonus
  & \makecell[l]{$\sum_{f}(t_{f}^{\rm air}-\bar{t})\,\mathbf{1}[\text{first contact}]$\\ $\bar{t}{=}0.4$\,s}
  & $0.75$ & $0.25$ & $0.75$ \\
Feet slide
  & $\sum_{f}\lVert\mathbf{v}_{f}\rVert\,\mathbf{1}[\text{contact}]$
  & $-0.1$ & $-0.1$ & $-0.1$ \\
Flat orientation & $\lVert\mathbf{g}_{xy}\rVert^{2}$ & $-1.0$ & $-1.0$ & $-1.0$ \\
Joint-position limits
  & \makecell[l]{$\sum_{j}\max(0,q_{j}-q_{j}^{\max})$\\ $\;+\max(0,q_{j}^{\min}-q_{j})$}
  & $-1.0$ & $-1.0$ & $-1.0$ \\
Undesired contacts
  & \makecell[l]{$\sum_{b}\mathbf{1}[\lVert\mathbf{f}_{b}\rVert>1\,\text{N}]$\\ (torso/arms)}
  & $-1.0$ & --- & $-1.0$ \\
Termination & $\mathbf{1}[\text{episode terminates}]$ & $-200$ & $-200$ & $-200$ \\
Joint deviation (hip)
  & $\sum_{j\in\text{hip}}(q_{j}-q_{j}^{\rm def})^{2}$
  & --- & $-0.1$ & --- \\
Joint deviation (arms)
  & $\sum_{j\in\text{arms}}(q_{j}-q_{j}^{\rm def})^{2}$
  & --- & $-0.1$ & --- \\
Joint deviation (fingers)
  & $\sum_{j\in\text{fingers}}(q_{j}-q_{j}^{\rm def})^{2}$
  & --- & $-0.05$ & --- \\
Joint deviation (torso)
  & $(q_{\rm torso}-q_{\rm torso}^{\rm def})^{2}$
  & --- & $-0.1$ & --- \\
Ball proximity
  & $\exp(-d_{\rm ball}^{2}/\sigma^{2})$, $\sigma{=}2.0$\,m
  & --- & --- & $-5.0$ \\
Ball collision
  & $\mathbf{1}[\lVert\mathbf{f}_{\rm ball}\rVert>0.1\,\text{N}]$
  & --- & --- & $-1.0$ \\
\bottomrule
\end{tabular}
\end{table}

The policy $\pi$ for all tasks is an MLP with
% \ruic{complete robot policy configuration details, size, num of layers, output dim, etc.}.
feed-forward layers with ELU activations and a Gaussian action head
(initial scalar standard deviation $1.0$, no observation normalization).
The hidden sizes are $[256,128,128]$ for \textsc{G1Flat} and
\textsc{G1Collision} and $[512,256,128]$ for \textsc{G1Rough}; the output
dimension matches the action space ($37$), and joint-position targets are
recovered as $\mathbf{q}^{\rm target}=\mathbf{q}^{\rm def}+0.5\,\mathbf{a}$
(scale $0.5$, default-offset on, no action clipping).
The policy is trained with PPO for $1000$ steps with $2048$ parallel agents in IsaacLab.
The hyperparameter used are listed in \Cref{tab:sim_ppo}.

% \ruic{small table summarizing all robot policy training hyperparams. mention actor critic obs are symmetric for simulation inference.}

\begin{table}[h]
\centering
\small
\caption{PPO hyperparameters for the simulation policies. All three tasks use
the same algorithmic settings; only the network width and the number of
training iterations vary per task. For all simulation tasks the actor and the
critic consume the \emph{same} observation group (symmetric actor--critic).}
\label{tab:sim_ppo}
\begin{tabular}{ll}
\toprule
Optimiser & Adam, lr $1{\times}10^{-3}$, adaptive schedule (target KL $0.01$) \\
Discount $\gamma$, GAE $\lambda_{\rm GAE}$ & $0.99$,\; $0.95$ \\
PPO clip / value-loss coef / entropy coef & $0.2$ / $1.0$ / $0.008$ \\
Epochs per update / minibatches per update & $5$ / $4$ \\
Steps per environment per rollout & $24$ \\
Parallel environments & $4096$ \\
Max gradient norm & $1.0$ \\
Iterations (\textsc{G1Flat} / \textsc{G1Rough} / \textsc{G1Collision}) & $1500$ / $3000$ / $1500$ \\
Random seed & $42$ \\
Actor / critic observations & symmetric (no privileged terms) \\
\bottomrule
\end{tabular}
\end{table}

\subsubsection{Policy Training for Safety Value Inference in the Real World}

For hardware experiments in \Cref{sec:hardware_quantitatives}, we use a Unitree G1 humanoid robot with 29 DoF to be consistent with the physical platform we have access to.
The environment setup for \textsc{G1Flat} and \textsc{G1Collision} is shown below.

% \ruic{follow above section and present 29 dof observation space, rewards, policy training hyperparameters}

% Replaces: \ruic{follow above section and present 29 dof observation space,
%   rewards, policy training hyperparameters}
% Location: Sec. 3.2 (Policy Training for Safety Value Inference in the Real
%   World), immediately after "The environment setup for \textsc{G1Flat} and
%   \textsc{G1Collision} is shown below."
% Note: Only \textsc{G1Flat} and \textsc{G1Collision} are deployed on hardware;
%   no 29-DoF \textsc{G1Rough} variant exists.
% Requires: \usepackage{makecell} (and booktabs).

\paragraph{Observation space.}
For deployment the observation history is stacked over the last $5$ control
steps (history length $5$, control rate $50$\,Hz). Table~\ref{tab:hw_obs}
lists the per-frame composition and the resulting actor/critic dimensions.
The actor uses only sensors that are also available on the physical platform
(no base linear velocity), while the critic additionally observes
$\mathbf{v}\in\mathbb{R}^{3}$ per frame: an \emph{asymmetric}
actor--critic.

\begin{table}[h]
\centering
\small
\caption{Observation composition for the $29$-DoF hardware policies. The
critic additionally consumes the base linear velocity $\mathbf{v}$ per frame
(privileged) and is therefore $3$ dimensions per frame larger than the
actor.}
\label{tab:hw_obs}
\begin{tabular}{lcc}
\toprule
Observation segment (per frame) & \textsc{G1Flat} & \textsc{G1Collision} \\
\midrule
Base angular velocity $\boldsymbol{\omega}\in\mathbb{R}^{3}$  & 3 & 3 \\
Projected gravity $\mathbf{g}\in\mathbb{R}^{3}$               & 3 & 3 \\
Velocity command $\mathbf{c}\in\mathbb{R}^{3}$                & 3 & 3 \\
Joint positions $\mathbf{q}\in\mathbb{R}^{29}$                & 29 & 29 \\
Joint velocities $\dot{\mathbf{q}}\in\mathbb{R}^{29}$         & 29 & 29 \\
Last action $\mathbf{a}\in\mathbb{R}^{29}$                    & 29 & 29 \\
Ball position in base frame $\mathbf{p}_{\rm ball}\in\mathbb{R}^{3}$ & --- & 3 \\
Ball velocity in base frame $\mathbf{v}_{\rm ball}\in\mathbb{R}^{3}$ & --- & 3 \\
\midrule
Per-frame dimension                                  & 96  & 102 \\
Actor observation ($5$-frame stack)                  & 480 & 510 \\
Critic observation (actor + $\mathbf{v}$/frame)      & 495 & 525 \\
\bottomrule
\end{tabular}
\end{table}

\paragraph{Rewards.}
Table~\ref{tab:hw_rewards} lists the reward terms used for training the two
hardware policies. \textsc{G1Collision} (HW) is warm-started from the
deployed \textsc{G1Flat} (HW) baseline; the only additions in the avoidance
policy are the three ball-related terms.

\begin{table}[h]
\centering
\small
\setlength{\tabcolsep}{4pt}
\renewcommand{\arraystretch}{1.15}
\caption{Reward composition for the $29$-DoF hardware policies. A dash (---)
denotes a term not used for that task. Termination conditions: time-out
($20$\,s), base height $<0.2$\,m, or projected-gravity tilt $>0.8$\,rad.}
\label{tab:hw_rewards}
\begin{tabular}{@{}l l c c@{}}
\toprule
Term & Formula & \textsc{G1Flat} (HW) & \textsc{G1Collision} (HW) \\
\midrule
Lin.\ vel.\ tracking
  & \makecell[l]{$\exp(-\lVert\mathbf{v}_{xy}-\mathbf{c}_{xy}\rVert^{2}/\sigma^{2})$\\ $\sigma{=}0.5$}
  & $1.0$ & $1.0$ \\
Ang.\ vel.\ tracking
  & \makecell[l]{$\exp(-(\omega_{z}-c_{\omega_{z}})^{2}/\sigma^{2})$\\ $\sigma{=}0.5$}
  & $0.5$ & $0.5$ \\
Alive bonus & $1$ per step & $0.15$ & $0.15$ \\
Flat orientation & $\lVert\mathbf{g}_{xy}\rVert^{2}$ & $-5.0$ & $-5.0$ \\
Joint-position limits
  & \makecell[l]{$\sum_{j}\max(0,q_{j}-q_{j}^{\max})$\\ $\;+\max(0,q_{j}^{\min}-q_{j})$}
  & $-5.0$ & $-5.0$ \\
Undesired contacts
  & $\sum_{b}\mathbf{1}[\lVert\mathbf{f}_{b}\rVert>1\,\text{N}]$ (excl.\ ankles)
  & $-1.0$ & $-1.0$ \\
Feet slide
  & $\sum_{f}\lVert\mathbf{v}_{f}\rVert\,\mathbf{1}[\text{contact}]$
  & $-0.2$ & $-0.2$ \\
Joint deviation (arms)
  & $\sum_{j\in\text{arms}}(q_{j}-q_{j}^{\rm def})^{2}$
  & $-0.1$ & $-0.1$ \\
Joint deviation (waist)
  & $\sum_{j\in\text{waist}}(q_{j}-q_{j}^{\rm def})^{2}$
  & $-1.0$ & $-1.0$ \\
Joint deviation (hip yaw/roll)
  & $\sum_{j\in\text{hip}}(q_{j}-q_{j}^{\rm def})^{2}$
  & $-1.0$ & $-1.0$ \\
$z$-velocity penalty & $v_{z}^{2}$ & $-2.0$ & $-2.0$ \\
Roll/pitch penalty & $\lVert\boldsymbol{\omega}_{xy}\rVert^{2}$ & $-0.05$ & $-0.05$ \\
Joint velocity & $\lVert\dot{\mathbf{q}}\rVert^{2}$ & $-10^{-3}$ & $-10^{-3}$ \\
Joint acceleration & $\lVert\ddot{\mathbf{q}}\rVert^{2}$
  & $-2.5{\times}10^{-7}$ & $-2.5{\times}10^{-7}$ \\
Action rate & $\lVert\mathbf{a}_{t}-\mathbf{a}_{t-1}\rVert^{2}$ & $-0.05$ & $-0.05$ \\
Energy (Unitree) & $\sum_{j}\tau_{j}\,\dot{q}_{j}$
  & $-2{\times}10^{-5}$ & $-2{\times}10^{-5}$ \\
Base height & $(h_{\rm base}-0.78)^{2}$ & $-10.0$ & $-10.0$ \\
Feet gait (Unitree)
  & \makecell[l]{period $0.8$\,s, offset $[0,0.5]$,\\ threshold $0.55$}
  & $0.5$ & $0.5$ \\
Feet clearance (swing)
  & $\sum_{f}(h_{f}-0.1)^{2}\,\mathbf{1}[\text{swing}]$
  & $1.0$ & $1.0$ \\
Ball proximity
  & $\exp(-d_{\rm ball}^{2}/\sigma^{2})$, $\sigma{=}2.0$\,m
  & --- & $-5.0$ \\
Ball collision
  & $\mathbf{1}[\lVert\mathbf{f}_{\rm ball}\rVert>0.1\,\text{N}]$
  & --- & $-20.0$ \\
Ball approaching
  & $\max(0,-\dot{d}_{\rm ball})\cdot\,\mathbf{1}[d_{\rm ball}\le 2\,\text{m}]$
  & --- & $-2.0$ \\
\bottomrule
\end{tabular}
\end{table}

\paragraph{PPO hyperparameters.}
Table~\ref{tab:hw_ppo} lists the PPO settings for the hardware policies. Both
tasks use a $[512,256,128]$ MLP with ELU activations and a Gaussian action
head (init scalar std $1.0$). Compared to the simulation lineage, the
hardware policies use (i) wider networks, (ii) an asymmetric actor--critic,
(iii) a smaller action scale ($\mathbf{q}^{\rm target}=\mathbf{q}^{\rm def}
+0.25\,\mathbf{a}$, vs.\ $0.5$ in simulation), and (iv) a richer domain
randomisation regime (friction sampling, base-mass perturbation, and a
periodic root-velocity push event) to bridge to hardware.

\begin{table}[h]
\centering
\small
\setlength{\tabcolsep}{4pt}
\renewcommand{\arraystretch}{1.15}
\caption{PPO hyperparameters for the $29$-DoF hardware policies.}
\label{tab:hw_ppo}
\begin{tabular}{@{}l l@{}}
\toprule
Hyperparameter & Value \\
\midrule
Actor/critic hidden dimensions & $[512,256,128]$ \\
Activation; action head & ELU; Gaussian, init scalar std $1.0$ \\
Observation normalisation & off (empirical normalisation on) \\
Optimiser
  & \makecell[l]{Adam, lr $10^{-3}$,\\ adaptive schedule (target KL $0.01$)} \\
Discount $\gamma$, advantage $\lambda_{\rm GAE}$ & $0.99$,\; $0.95$ \\
PPO clip / value-loss coef / entropy coef & $0.2$\;/\;$1.0$\;/\;$0.01$ \\
Epochs per update / minibatches per update & $5$\;/\;$4$ \\
Steps per environment per rollout & $24$ \\
Parallel envs (\textsc{G1Flat} / \textsc{G1Collision}) & $8192$\;/\;$4096$ \\
Iterations (\textsc{G1Flat})
  & \makecell[l]{$30\,000$\\ (deployed @ $6\,000$)} \\
Iterations (\textsc{G1Collision})
  & \makecell[l]{$5\,000$\\ (deployed @ $4\,999$)} \\
Warm start (\textsc{G1Collision}) & resume from deployed \textsc{G1Flat} baseline \\
Observation history length & $5$ frames \\
Actor / critic observations
  & \textbf{asymmetric}; critic adds $\mathbf{v}\in\mathbb{R}^{3}$ per frame \\
Domain randomisation
  & \makecell[l]{friction $\in[0.3,1.0]$;\\
                 base mass $+[-1,3]$\,kg;\\
                 push every $5$\,s,\\ $\;v_{x,y}\sim\mathcal{U}[-0.5,0.5]$\,m/s} \\
Random seed & $42$ \\
\bottomrule
\end{tabular}
\end{table}

% \ruic{make sure to be clear about difference in the ppo setting for hardware from simulation (e.g., asymmetrical actor critic observation to accommodate hardware data availability)}

\paragraph{Sim-vs-hardware differences in PPO.}
The hardware lineage diverges from the simulation lineage in four respects.
\emph{(i) Asymmetric actor--critic.} Because the base linear velocity
$\mathbf{v}$ is not reliably observable on the deployed platform, it is
withheld from the actor and provided only to the critic; the actor consumes
$96$ (\textsc{G1Flat}) or $102$ (\textsc{G1Collision}) dimensions per frame
while the critic consumes $99$ and $105$, respectively. \emph{(ii) Observation
history.} Hardware observations are stacked over the last $5$ frames
(simulation uses single-frame observations). \emph{(iii) Action scale.}
Joint-position targets are recovered as $\mathbf{q}^{\rm target}=\mathbf{q}^{\rm def}+0.25\,\mathbf{a}$
on hardware versus $0.5\,\mathbf{a}$ in simulation, halving the action range
to reduce abrupt commands. \emph{(iv) Domain randomisation.} Friction is sampled
in $[0.3,1.0]$, the torso mass is perturbed by $[-1,3]$\,kg, and a periodic
root-velocity push event with $v_{x,y}\!\sim\!\mathcal{U}[-0.5,0.5]$\,m/s is
applied every $5$\,s; the simulation lineage uses no such randomisation.
All other PPO settings (optimiser, $\gamma$, $\lambda_{\rm GAE}$, clip,
value-loss/entropy coefficients, minibatches) are identical to the
simulation policies.

\subsection{Data Collection for Safety Value Learning}\label{append:data_for_safety_value_learning}

With robot policy $\pi$, the dataset $\cD$ for safety value learning is generated by evaluating $\pi$ in each simulation environment with $2048$ agents for $1000$ steps.
For simulation inference tasks in \Cref{sec:sim_quantitatives}, safety signals are labeled following \Cref{sec:exp_task}.
For hardware inference in \Cref{sec:hardware_quantitatives}, we label the safety signal for \textsc{G1Flat} as $\ell_\mathrm{flat}\defeq \ell_\mathrm{ori}$.
The contact component $\ell_\mathrm{contact}$ in the safety signal for \textsc{G1Collision} is synthesized as $\ell_\mathrm{contact}\defeq 1_{d_\mathrm{ball}\leq 0.3} - 1_{d_\mathrm{ball} > 0.3}$.

\subsection{Safety Value Learning}

Using the dataset $\cD$, we train the $\lambda$-Reachability approach, its ablations, and baselines.
All safety values are represented as neural networks with
% \ruic{fill in network size}.
two hidden layers of width $256$ and ReLU activations, mapping each task's
observation to a scalar value $V^{\pi}(x)\in\mathbb{R}$. The final linear
layer is initialised with zero weights and bias $-2$, so that $V^{\pi}$ starts
strongly safe everywhere. For $\lambda$-Reachability we instantiate two
critics $V^{\pi}_{\theta_{1}}$, $V^{\pi}_{\theta_{2}}$ with this architecture
and report the ensemble mean
$V^{\pi}(x)=\tfrac{1}{2}\!\bigl(V^{\pi}_{\theta_{1}}(x)+V^{\pi}_{\theta_{2}}(x)\bigr)$;
the baselines use a single critic with the same architecture.
Training is done following \Cref{sec:safety_value_training_main}.
Full training parameters for $\lambda$-Reachability and DPE are reported below.

% \ruic{table summarzing all safety value learning hyperparametsr, such as lr, delta, terminal value, Batch, n\_max, tau, etc.}

% Replaces: \ruic{table summarzing all safety value learning hyperparametsr,
%   such as lr, delta, terminal value, Batch, n\_max, tau, etc.}
% Location: Sec. Safety Value Learning, after
%   "Full training parameters for $\lambda$-Reachability and DPE are reported below."
% Requires: \usepackage{makecell} (and booktabs).
\begin{table}[h]
\centering
\small
\setlength{\tabcolsep}{4pt}
\renewcommand{\arraystretch}{1.15}
\caption{Safety-value learning hyperparameters. The main columns report the
configuration used for the hardware-inference lineage; values that change for
the simulation-inference lineage are listed in the last column. For
$\lambda$-Reachability we additionally report ablations
$\lambda\in\{0.99,0.95,0.50,0\}$ in \Cref{sec:sim_quantitatives}; the value
shown here is the default ($\lambda{=}0.99$) used in all main results.}
\label{tab:value_hparams}
\begin{tabular}{@{}l l l l@{}}
\toprule
Hyperparameter & $\lambda$-Reachability & DPE & Sim-inference override \\
\midrule
Critic ensemble
  & \makecell[l]{$2$\\ (clipped double-$Q$)}
  & $1$
  & --- \\
Horizon discount
  & $\lambda=0.99$ (main)
  & \makecell[l]{$\gamma$ annealed\\ $0.9\!\to\!0.99$}
  & --- \\
Per-step survival $\delta$ & $0.99$ & --- & --- \\
Absorption value $v_{\rm term}$ & $-10^{6}$ & --- & --- \\
Max bootstrap horizon $n_{\max}$
  & $200$
  & \makecell[l]{$1$\\ (one-step TD)}
  & --- \\
Hidden dims, activation
  & $[256,256]$, ReLU
  & $[256,256]$, ReLU
  & --- \\
Optimiser, learning rate $\alpha$
  & Adam, $10^{-3}$
  & Adam, $10^{-3}$
  & --- \\
Batch size $B$ & $256$ & $256$ & --- \\
Target Polyak rate $\tau$
  & $0.005$
  & $0.05$
  & \makecell[l]{$\tau{=}0.05$\\ ($\lambda$-reach)} \\
Target update period
  & every step
  & every $10$ steps
  & \makecell[l]{every $10$ steps\\ ($\lambda$-reach)} \\
Loss
  & \makecell[l]{regression\\ $+$ lower-bound hinge\\ $+$ monotonicity hinge\\ $+$ sign BCE ($\alpha_{\rm bce}{=}5$);\\ weights $(1,1,1,1)$}
  & \makecell[l]{MSE on\\ $(1{-}\gamma)\ell$\\ $\;+\,\gamma\max(\ell,V^{-}(x'))$}
  & \makecell[l]{hinge / BCE\\ weights $\to 0$} \\
Total gradient steps & $10\,000$ & $10\,000$ & $2\,000$ \\
Evaluation cadence & every $500$ steps & every $500$ steps & --- \\
Random seed & $0$ & $0$ & --- \\
\bottomrule
\end{tabular}
\end{table}

\subsection{Safety Value Inference}\label{append:safety_value_inference}

% \ruic{push tilt only; contact in real is synthesized by mocap. other signals not available, e.g., root height since it depends on scan point)}.

% \ruic{sim2sim verify, hardware setup, mocap setup, communication, etc.}

Simulation inference is performed by applying the learned $V^\pi$ to the same environments on which the rollout policy was trained, and computing metrics according to \Cref{sec:safety_value_inference_metrics}.
For hardware inference, we collect real data from policy rollouts with artificially generated disturbances.
To record observations for \textsc{G1Collision}, we use an Optitrack motion capture system to stream the poses of both the humanoid base and a ball, which serves as the dynamic obstacle.
The relative position and velocity of the ball in the humanoid base frame are computed online as part of the observation for the safety value function.
The safety values are applied to the collected trajectories and evaluated against labeled safety signals following \Cref{sec:safety_value_inference_metrics}.

\newpage
\section{Hardware Experiment Results}\label{append:additional_hw_results}

In this section, we present additional hardware experimental results on \textsc{G1Flat} and \textsc{G1Collision}, including video screenshots and plots of the safety signal $\ell$ and the inferred safety value $V^\pi$.
In each plot, we mark the timestamp of the external events that perturb the robot towards unsafe regions and observe the first timestamp at which $\lambda$-Reachability and the DPE baseline predict non-recoverable states, and the first timestamp of actual unsafe events.
It is evident that $\lambda$ outperforms the baseline via earlier warning and more accurate value prediction.

% \ruic{add segment video screenshots and plots. key frames should be marked.}
\begin{figure*}[h]
    \centering
    \includegraphics[width=0.99\textwidth]{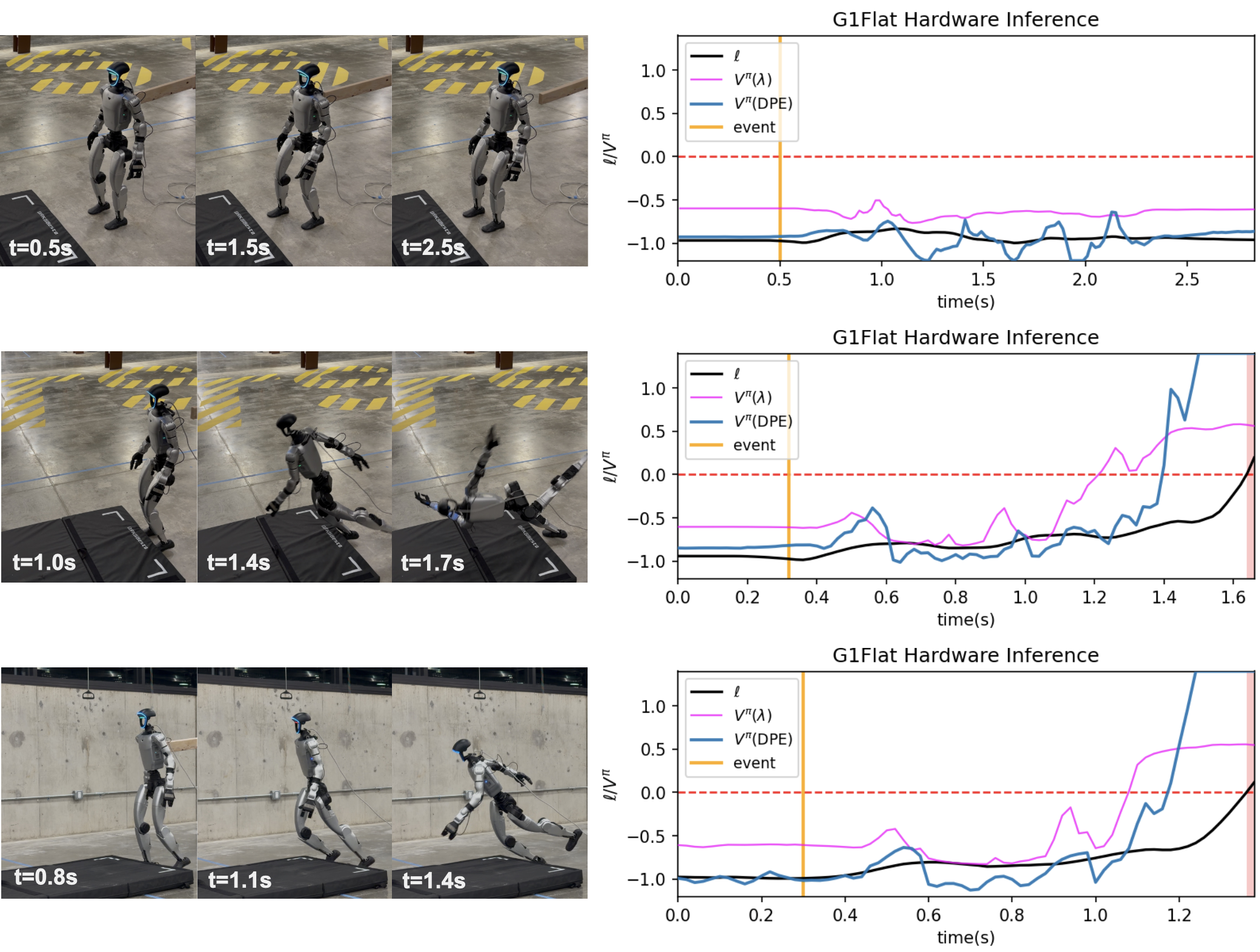}

    \vspace{0.4em}

    \includegraphics[width=0.99\textwidth]{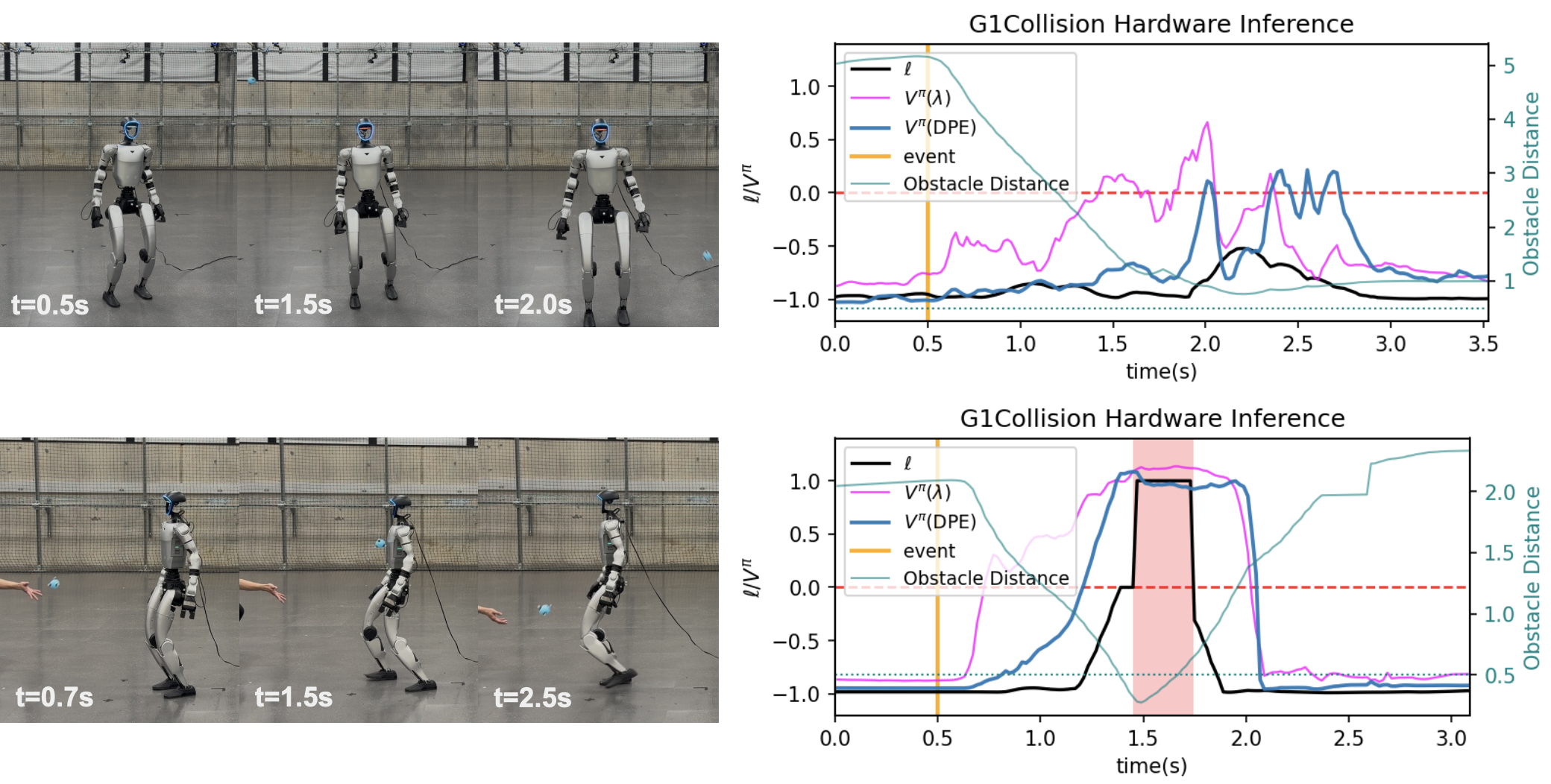}

    \caption{Additional hardware experiment results.}
    \label{fig:hardware_results}
\end{figure*}

\end{document}